\newif\ifpreprint
\preprinttrue

\newif\ifrevision
\revisionfalse
\ifrevision
\preprintfalse
\fi

\newif\ifrc
\rctrue

\ifpreprint
\documentclass[sn-basic,pdflatex]{sn-jnl}
\else
\documentclass[sn-basic,pdflatex]{sn-jnl}
\fi

\usepackage[utf8]{inputenc}
\usepackage[british]{babel}
\usepackage{algorithm}
\usepackage{hyperref}
\usepackage{caption}
\usepackage{subcaption}
\usepackage{textcomp}
\usepackage{mathtools}
\usepackage{booktabs}
\usepackage{nicefrac}
\usepackage{tablefootnote}
\usepackage{tabularx}
\usepackage{amssymb}
\usepackage{pifont}
\usepackage{supertabular}
\usepackage{placeins}
\usepackage{wrapfig}
\usepackage[nolist]{acronym}
\usepackage{threeparttable}

\usepackage{amsmath,amssymb,amsfonts,amsthm}
\usepackage{graphicx}
\usepackage[dvipsnames]{xcolor}
\usepackage{textcomp}
\usepackage{cleveref}

\makeatletter 
\def\@acknow{}%
\long\def\EarlyAcknow#1 \par{%
\def\@acknow{\abstractfont\abstracthead*{Acknowledgments}
#1\par}}%

\def\printabstract{
\ifx\@acknow\empty\else\@acknow\fi\par%
\ifx\@abstract\empty\else\@abstract\fi\par
}
\makeatother

\jyear{2023}%

\DeclareMathOperator*{\argmin}{\arg\!\min}

\ifrevision

\newcommand{\reviewerA}{\color{RoyalBlue}}
\newcommand{\reviewerB}{\color{ForestGreen}}

\else

\newcommand{\reviewerA}{\color{black}}
\newcommand{\reviewerB}{\color{black}}

\fi
\newcommand{\erevision}{\color{black}}

\newcommand{\eneedsrevision}{\color{black}}

\newtheorem{theorem}{Theorem}
\newtheorem{definition}{Definition}
\theoremstyle{definition}
\newtheorem*{derivation*}{Derivation}
\newtheorem*{example}{Example}

\usepackage{stackengine}
\newcommand\xrowht[2][0]{\addstackgap[.5\dimexpr#2\relax]{\vphantom{#1}}}

\begin{document}

\begin{acronym}
\acro{name}[ABCD]{Adaptive Bernstein Change Detector}
\acro{pca}[PCA]{Principal Component Analysis}
\end{acronym}

\title{Adaptive Bernstein Change Detector for High-Dimensional Data Streams}

\author*{\fnm{Marco} \sur{Heyden}*} \email{marco.heyden@kit.edu}
\author{\fnm{Edouard} \sur{Fouch\'{e}}} \author{\fnm{Vadim} \sur{Arzamasov}}
\author{\fnm{Tanja} \sur{Fenn}}
\author{\fnm{Florian} \sur{Kalinke}}
\author{\fnm{Klemens} \sur{B\"{o}hm}} 
\email{firstname.lastname@kit.edu}

\affil{\orgname{Karlsruhe Institute of Technology}, \orgaddress{\city{Karlsruhe}, \country{Germany}}}

\abstract{
Change detection is of fundamental importance when analyzing data streams. Detecting changes both quickly and accurately enables monitoring and prediction systems to react, e.g., by issuing an alarm or by updating a learning algorithm. However, detecting changes is challenging when observations are high-dimensional. In high-dimensional data, change detectors should not only be able to identify when changes happen, but also in which subspace they occur. Ideally, one should also quantify how severe they are. 
Our approach, ABCD, has these properties. ABCD learns an encoder-decoder model and monitors its accuracy over a window of adaptive size. ABCD derives a change score based on Bernstein's inequality to detect deviations in terms of accuracy, which indicate changes. Our experiments demonstrate that ABCD outperforms its best competitor by up to 20\% in F1-score on average. It can also accurately estimate changes' subspace, together with a severity measure that correlates with the ground truth. 
}

\EarlyAcknow{
This work was supported by the German Research Foundation (DFG) Research Training Group GRK 2153:
\textit{Energy Status Data --- Informatics Methods for its Collection, Analysis and Exploitation.}
\newline

This version of the article has been accepted for publication, after peer review but is not the Version of Record and does not reflect post-acceptance improvements, or any
corrections. The Version of Record is available online at: \href{https://doi.org/10.1007/s10618-023-00999-5}{https://doi.org/10.1007/s10618-023-00999-5}.
}


\maketitle

\section{Introduction}\label{sec:introduction}
Data streams are open-ended, ever-evolving sequences of observations from some process. They pose unique challenges for analysis and decision-making. One crucial task is to detect changes, i.e., shifts in the observed data, that may indicate a change in the underlying process. Change detection has been an active research area. However, the high-dimensional setting, in which observations contain a large number of simultaneously measured quantities, did not receive enough attention. \reviewerB Yet, it may yield useful insights in environmental monitoring~\citep{dejong_UnsupervisedChangeDetection_2019}, human activity recognition~\citep{vrigkas_ReviewHumanActivity_2015}, network traffic monitoring~\citep{naseer_LearningRepresentationsNetwork_2020}, automotive~\citep{liu_HighdimensionalDataAbnormity_2019}, predictive maintenance~\citep{xiaopingzhaojiaxinwu_FaultDiagnosisMotor_2018}, and biochemical engineering~\citep{mowbray_MachineLearningBiochemical_2021}: 

\begin{example}[Biofuel production]
    The production of fuel from biomass is a complex process comprising many interdependent process steps. Those include pyrolysis, synthesis, distillation, and separation. Many steps rely on (by-)products of other steps as reactants, leading to a highly interconnected system with many process parameters. A monitoring system tracks the process parameters to detect failures in the plant: (i) The system must detect changes in a large (i.e., high-dimensional) vector of process parameters, which may indicate failures. (ii) The system must find out which process parameters are affected by the change to allow for a targeted reaction. Since the system is very complex and has many interconnected components, change is often evident only when considering correlations between process parameters. An example would be the correlation between temperature and concentration fluctuations. So it is insufficient to monitor each process parameter in isolation. (iii) There can exist slight changes which only require minor adjustments and more severe ones that require immediate intervention to avoid a shutdown of the plant. The monitoring system should provide an estimate of the severity of change. 
\end{example}
\erevision

The example illustrates three requirements for modern change detectors:
\begin{itemize}
    \item \textbf{R1: Change point.} The primary task of change detectors is to identify that the data stream has changed and when it occurred.
    \item \textbf{R2: Change subspace.} 
    A change may only concern a subset of dimensions~--- the \emph{change subspace}. Change detectors for high-dimensional data streams should be able to identify such subspaces.
    \item \textbf{R3: Change severity.} 
    Quantifying relative change severity to distinguish between changes of different importance is essential to react appropriately. 
\end{itemize}

\reviewerB
Prior works already acknowledge the relevance of the above requirements~\citep{lu_learning_2019, webb_analyzing_2018}.
\erevision
However, fulfilling R1--R3 in combination remains challenging since they depend on each other: on the one hand, detecting changes in high-dimensional data is difficult because changes typically only affect few dimensions. Unaffected dimensions ``dilute'' a change (i.e., a change occuring in a subspace appears to be less severe in the full space). This might make changes harder to detect in all dimensions. 
On the other hand, detecting the change subspace should occur \emph{after} detecting a change, since monitoring all possibles subspaces is intractable. Last, one should restrict computation of change severity to the change subspace to eliminate dilution.

Existing methods for change detection, summarized in \Cref{tab:related-approaches}, either are univariate (UV), multivariate (MV), or specifically designed for high-dimensional data (HD); the latter claim efficiency w.r.t. high-dimensionality or resilience against the ``curse of dimensionality''. However, they do not fulfill R1-R3 in combination sufficiently well as \Cref{sec:relatedwork} describes. 

Thus, we propose the \acf{name}, which addresses R1-R3 in combination. We articulate our contributions as follows: 

\textbf{(i)~Problem Definition:} We formalize the problem of detecting changes in high-dimensional data streams such that R1-R3 can be tackled in combination.
\textbf{(ii)~\acl{name}:} We present ABCD, a change detector for high-dimensional data, that satisfies R1--R3. 
\reviewerB
It monitors the loss of an encoder-decoder model using an adaptive window size and statistical testing. Adaptive windows enable ABCD to detect severe changes quickly and, over a longer period, identify hard-to-detect changes that would typically require a large window size. 
\erevision
\textbf{(iii)~Bernstein change score:} Our approach applies a statistical test based on Bernstein's inequality. This limits the probability of false alarms. 
\reviewerB
\textbf{(iv)~Online computation:} We propose an efficient method for computing the change score in adaptive windows and discuss design choices leading to constant time and memory.
\erevision
\textbf{(v)~Benchmarking:}
\reviewerB
We conduct experiments on 10 data streams based on real-world and synthetic data with many dimensions and compare \ac{name} with recent approaches. The results indicate that ABCD outperforms its competitors consistently w.r.t. R1--R3, is robust to high-dimensional data and is useful in domains including human activity recognition, gas detection, and image processing. 
\erevision
We also study \ac{name}'s parameter sensitivity. Our code\footnote{\url{https://github.com/heymarco/AdaptiveBernsteinChangeDetector}} follows the popular Scikit-Multiflow API~\citep{montiel_scikit-multiflow_2018}, so it is easy to use in future research.

\section{Related work} \label{sec:relatedwork}

\begin{table}[tb]
    \centering
    \caption{Related work.}
    \begin{tabular}{llcccc}
    \toprule
         Approach & Reference & Type & R1 & R2 & R3 \\
         \midrule
         ADWIN & \cite{bifet_learning_2007} & UV & \checkmark & -- & --\\
         SeqDrift2 & \cite{pears_detecting_2014}& UV & \checkmark & -- & -- \\
         kdq-Tree & \cite{dasu_information-theoretic_2006}& MV & \checkmark &--&\checkmark\\
         PCA-CD & \cite{qahtan_pca-based_2015} & MV & \checkmark & --&\checkmark\\
         IKS & \cite{dos_reis_fast_2016}& MV & \checkmark & \checkmark & -- \\ 
         LDD-DSDA & \cite{liu_regional_2017}& MV & \checkmark & -- & --\\
         AdwinK & \cite{faithfull_combining_2019}& MV & \checkmark  &\checkmark& --\\
         D3 & \cite{gozuacik_unsupervised_2019} & MV & \checkmark & -- & \checkmark \\
         ECHAD & \cite{ceci_ECHADEmbeddingbasedChange_2020}& MV & \checkmark & --& \checkmark \\
         IBDD & \cite{souza_unsupervised_2020}& HD & \checkmark & --& \checkmark \\
         WATCH & \cite{faber_watch_2021}& HD & \checkmark & --& \checkmark \\ 
         \textbf{\ac{name}} & this work& HD & \checkmark &\checkmark&\checkmark\\
         \bottomrule
    \end{tabular}
    \label{tab:related-approaches}
\end{table}

\subsection{Change detector types} 

Most existing change detectors are \textit{supervised}, i.e., they focus on detecting changes in the relationship between input data and a target variable~\citep{iwashita_overview_2019}. However, class labels are rarely available in reality, which limits their applicability. 
On the contrary, the \emph{ unsupervised} change detectors aim to detect changes only in the input data. Our approach belongs to this category, so we restrict our review to unsupervised approaches.

Most existing approaches detect changes whenever a measure of discrepancy between newer observations (the current window) and older observations (the reference window) exceeds a threshold. Some approaches, e.g., D3~\citep{gozuacik_unsupervised_2019} or PCA-CD~\citep{qahtan_pca-based_2015}, implement the reference and current window as two contiguous sliding windows. Other approaches, such as IBDD~\citep{souza_unsupervised_2020}, IKS~\citep{dos_reis_fast_2016} or WATCH~\citep{faber_watch_2021} use a fixed reference window.
A major problem is to choose the appropriate size for the window;~ thus \citep{bifet_learning_2007} propose windows of adaptive size, that grow while the stream remains unchanged and shrink otherwise. Several work leverage this principle, e.g. \citep{sun_online_2016,khamassi_self-adaptive_2015,fouche_scaling_2019,suryawanshi_adaptive_2022}.
We also use adaptive windows to lower the number of parameters of \ac{name}.

\subsection{Univariate change detection} 

There exist many approaches for change detection in univariate (UV) data streams. Two of them, Adaptive Windowing (ADWIN) \citep{bifet_learning_2007} and SeqDrift2 \citep{pears_detecting_2014}, share some similarity with our approach. Like ADWIN, \ac{name} relies on an adaptive window. Like SeqDrift2, it uses Bernstein's inequality \citep{bernstein_modification_1924}. 
But unlike ADWIN and SeqDrift2, 
ABCD can handle high-dimensional data while fulfilling R1-R3.

\subsection{Multivariate change detection} \label{sec:multivariate}

To detect changes in multivariate (MV) 
data, some approaches apply univariate algorithms in each dimension of the stream. \cite{faithfull_combining_2019} propose to use one ADWIN detector per dimension (with $k$ dimensions). They declare a change whenever a certain fraction of the detectors agree. We call this approach AdwinK later on. Similarly, IKS \citep{dos_reis_fast_2016} uses an incremental variant of the Kolmogorov-Smirnov test deployed in each dimension. Unlike AdwinK, IKS issues an alarm if at least one dimension changes.

There also exist approaches specifically designed for multivariate \citep{jaworski_concept_2020, ceci_ECHADEmbeddingbasedChange_2020, qahtan_pca-based_2015, gozuacik_unsupervised_2019, dasu_information-theoretic_2006}, or even high-dimensional (HD) data \citep{faber_watch_2021, souza_unsupervised_2020}. 
\reviewerB 
Similar to ABCD, \cite{jaworski_concept_2020} and \cite {ceci_ECHADEmbeddingbasedChange_2020} use dimensionality-reduction methods to capture the relationships between dimensions. However, our approach is computationally more efficient, limits the probability of false alarms, identifies change subspace, and estimates change severity. 
\erevision
D3 \citep{gozuacik_unsupervised_2019} uses the AUC-ROC score of a discriminative classifier that tries to distinguish the data in two sliding windows. 
It reports a change if the AUC-ROC score exceeds a pre-defined threshold. PCA-CD~\citep{qahtan_pca-based_2015} first maps observations in two windows to fewer dimensions using PCA. Then the approach estimates the KL-divergence between both windows for each principal component. PCA-CD detects a change if the maximum observed KL-divergence exceeds a threshold. 
However, \citep{goldenberg_survey_2019} point out that this technique is limited to linear transformations and ignores combined change in multiple dimensions. 
LDD-DSDA \citep{liu_regional_2017} measures the degree of local drift that describes regional density changes in the input data. The approach proposed by \citep{dasu_information-theoretic_2006} structures observations from two windows (sliding or fixed) in a kdq-tree. 
For each node, they measure the KL-divergence between observations from both windows. However, \,\citep{qahtan_pca-based_2015} show experimentally that this approach is not suitable for high-dimensional data. 

IBDD~\citep{souza_unsupervised_2020} and WATCH~\citep{faber_watch_2021} specifically address challenges arising from high-dimensional data. The former monitors the mean squared deviation between two equally sized windows. The latter monitors the Wasserstein distance between a reference and a sliding window. However, both cannot detect change subspaces or measure severity. 

\subsection{Offline change point detection}

Offline change point detection, also known as signal segmentation, divides time series of a given length into $K$ homogeneous segments \citep{truong_selective_2020}. 
Many of the respective algorithms are not suitable for data streams:
Some require specifying $K$ a priori~\citep{bai_critical_2003, harchaoui_retrospective_2007, lung-yut-fong_homogeneity_2015}; others~\citep{killick_optimal_2012, lajugie_large-margin_2014, matteson_nonparametric_2014, chakar_robust_2017, garreau_consistent_2018} scale superlinearly with time. WATCH~\citep{faber_watch_2021}, discussed above, is the state of the art extension of offline change point detection to data streams.

\subsection{Change subspace}

The notion of a \emph{change subspace} is different from the existing notion of \emph{change region} \citep{lu_learning_2019}. The former describes a subset of dimensions that changed, the latter identifies density changes in some local region, e.g., a hyper-rectangle or cluster \citep{liu_regional_2017}. Our definition of change subspaces is related to \textit{marginal change magnitude} \citep{webb_analyzing_2018}, but is more general since it can also accomodate changes in a subspace's joint distribution.

Because high-dimensional spaces are typically sparse (due to the curse of dimensionality), identifying density changes in them is not effective. On the other hand, knowing that a change affected a specific set of dimensions can help identify the cause of the change, as we have motivated in our introductory example. Thus, we focus on detecting change subspaces in this work.

In the domain of statistical process control, some approaches extend well-known methods, such as Cusum~\citep{page_continuous_1954} or Shewhart charts~\citep{shewhart_economic_1931}, to multiple dimensions. They address the problem of identifying change subspaces to some extent, however, they often make unrealistic assumptions: they focus on Gaussian or sub-Gaussian data \citep{chaudhuri_sequential_2021, xie_sequential_2020}, require that different dimensions are initially independent \citep{chaudhuri_sequential_2021}, require subspace changes to be of low rank~\citep{xie_sequential_2020}, or assume that the size of the change subspace is known a priori \citep{jiao_subspace_2018}. 

From the approaches reviewed in \Cref{sec:multivariate} only AdwinK and IKS identify the corresponding change subspace. However, both approaches do not find changes that hide in subspaces, e.g., correlation changes, 
because they monitor each dimension in isolation. In contrast, our approach aims to learn the relationships between different dimensions so that it can detect such changes. Next, AdwinK cannot identify subspaces with fewer than $k$ dimensions.

\subsection{Change severity}

According to \citep{lu_learning_2019}, change severity is a positive measure of the discrepancy between the data observed before and after the change. 
One can either measure the divergence between distributions directly, as done by kdq-Tree\,\citep{dasu_information-theoretic_2006}, LDD-DSDA\,\citep{liu_regional_2017}, and WATCH\,\citep{faber_watch_2021}, or indirectly with a score that correlates with change severity, as done by D3\,\citep{gozuacik_unsupervised_2019}. Following this reasoning, an approach that satisfies R3 should compute a score that depends on the change severity \citep{gozuacik_unsupervised_2019, dasu_information-theoretic_2006, souza_unsupervised_2020, qahtan_pca-based_2015, faber_watch_2021}, i.e., the higher the score, the higher the severity. Finally, 
hypothesis-testing-based approaches, such as ADWIN\,\citep{bifet_learning_2007}, SeqDrift2\, \citep{pears_detecting_2014}, AdwinK\,\citep{ faithfull_combining_2019}, or IKS \citep{dos_reis_fast_2016}, do not quantify change severity: a slight change observed over a longer time can lead to the same $p$-value as a severe change observed over a shorter time, hence $p$ is not informative about change severity. 

\reviewerB
\subsection{Pattern based change detection}\label{sec:pbcd}

A related line of research, pattern-based change detection, deals with identifying changes in temporal graphs~\citep{loglisci_MiningMicroscopicMacroscopic_2018,impedovo_EfficientAccurateNonexhaustive_2019,impedovo_CondensedRepresentationsChanges_2020,impedovo_SimultaneousProcessDrift_2020}. In particular, \cite{loglisci_MiningMicroscopicMacroscopic_2018} detect changes in the graph, identify the affected subgraphs, and quantify the amount of change for these subgraphs. 
This is similar to our methodology. However, these methods work well with graph data, but we are dealing with vector data. To apply these methods in our context, one would need to create a graph, e.g., by representing each dimension as a node and indicating pairwise correlations with edges. However, constructing such a graph becomes impractical for high-dimensional observations because of the exponentially growing number of subspaces.
\eneedsrevision

\subsection{Competitors}
In our experiments, we compare to AdwinK, IKS, D3, IBDD, and WATCH. IBDD, WATCH, and D3 are recent change detectors for multivariate and high-dimensional data that fulfill R3. AdwinK extends the ADWIN algorithm to the multivariate case and fulfills R2. Finally, IKS is the only approach employing a non-parametric two-sample test for change detection while also satisfying R2. 

\section{Preliminaries} \label{sec:preliminaries}

We are interested in finding changes in the last $t$ observations $S=(x_1, x_2, \ldots, x_t)$ from a stream of data. Each $x_i$ is a $d$-dimensional vector independently drawn from a (unknown) distribution $F_i$. 
We assume without loss of generality that each vector coordinate is bounded in $[0,1]$, i.e., $x_i \in [0,1]^d$.

\begin{definition}[Change] A change occurs at time point $t^*$ if the data-generating distribution changes after $t^*$: $F_{t^*} \neq F_{t^*+1}$.  
\end{definition}

In high-dimensional data, changes typically affect only a subset of dimensions, which we call the \emph{change subspace}. Let $D=\{1,2,\dots,d\}$ be the set of dimensions and $F^{D'}_i$ be the joint distribution of $F_i$ observed in the subspace $D' \subseteq D$ at time step $i$. We define the change subspace as follows: 

\begin{definition}[Change subspace]
The change subspace $D^*$ at time $t^*$ is the union of all $D' \subseteq D$ in which the joint distribution $F^{D'}$ changed and which does not contain a subspace $D''$ for which $F^{D''}_{t^*} \neq F^{D''}_{t^*+1}$. 
\end{definition} 

If the dimensions in $D^*$ are uncorrelated, then changes will be visible on the marginal distributions, i.e., all $D'$ are of size~1. However, changes may only be detectable w.r.t the joint distribution of $D^*$ or the union of its subspaces of size greater than~1, which our definition accommodates. Note that the definition can also handle multiple co-occurring changes and considers them as one single change.
Last, change severity measures the difference between $F^{D^*}_{t^*}$ and $F^{D^*}_{t^*+1}$: 

\begin{definition}[Change severity]
The \textit{severity} of a change is a positive function $\Delta$ of the mismatch between $F^{D^*}_{t^*}$ and $F^{D^*}_{t^*+1}$. 
\end{definition}

Since we do not know the true distributions $F_{t^*}$ and $F_{t^*+1}$, the best we can do is detecting changes and their characteristics based on the observed data. 

\section{Approach} \label{sec:approach}

\subsection{Principle of ABCD}\label{sec:overview}

Direct comparison of high-dimensional distributions is impractical as it requires many samples~\citep{gretton_kernel_2012}. 
Yet the number of variables required to describe such data with high accuracy is often much smaller than $d$~\citep{lee_nonlinear_2007}. Dimensionality reduction techniques let us \textit{encode} observations in fewer dimensions. 
The more information encodings retain, the better one can reconstruct (\emph{decode}) the original data. However, if the distribution changes, the reconstruction will degrade and produce higher errors. 

\begin{figure}[t]
    \centering
    \includegraphics[width=\linewidth]{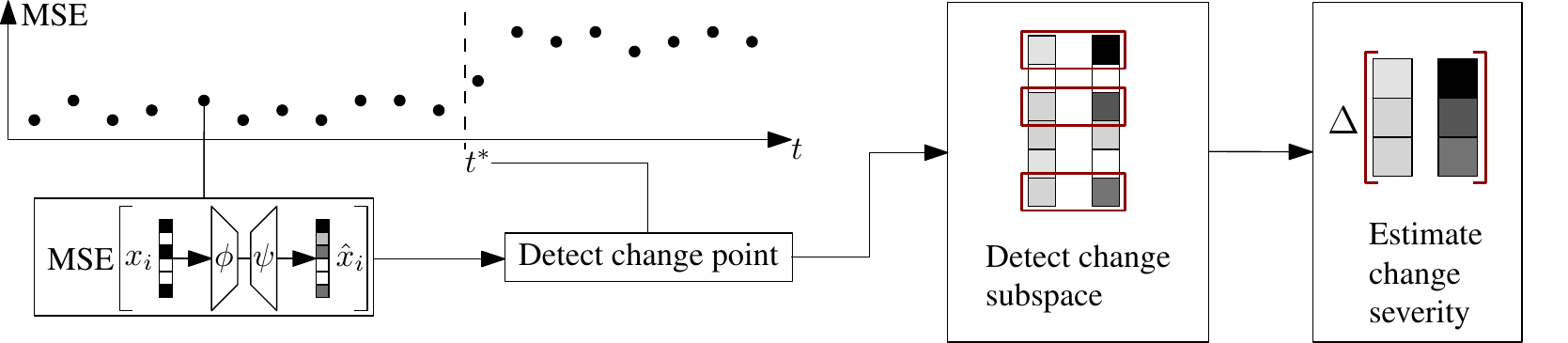}
    \caption{Overview of ABCD.}
    \label{fig:concept}
\end{figure}

We leverage this principle in ABCD by monitoring the reconstruction loss of an encoder-decoder model $\psi\circ\phi$ for some encoder function $\phi$ and decoder function $\psi$. \Cref{fig:concept} illustrates this. 
Specifically, we first learn $\phi: [0,1]^d \to [0,1]^{d'}$ with $d' = \lfloor\eta d\rfloor < d$, $\eta \in (\nicefrac{1}{d},1)$, mapping the data to fewer dimensions, and $\psi: [0,1]^{d'} \to [0,1]^d$. Then, we monitor the loss between each $x_t$ and its reconstruction $\hat{x}_t = \psi\circ\phi(x_t) = \psi(\phi(x_t))$:
\begin{equation}
    L_t = MSE(x_t, \hat{x}_t) = \frac{1}{d} \sum_{j=1}^{d} \left(x_{t,j} - \hat{x}_{t,j}\right)^2 = \frac{1}{d} \sum_{j=1}^d L_{t,j}
\end{equation} 

We hypothesize that distribution changes lead to outdated encoder-decoder models~--- see for example~\citep{jaworski_concept_2020} for empirical evidence. Hence, we assume that changes in the reconstruction affect the \emph{mean} $\mu_{t^*+1}$ of the loss, because the model can no longer accurately reconstruct the input:
\begin{equation}
\label{eq:implication}
F_{t^*} \neq F_{t^*+1} \implies \mu_{t^*} \neq \mu_{t^*+1}
\end{equation}

We can now replace the definition of change in high-dimensional data with an easier-to-evaluate, univariate proxy:
\begin{equation}\label{eq:concept-change-simplified}
    \exists t^* \in [1, \ldots, t]: \mu_{t^*} \neq \mu_{t^*+1}
\end{equation}
It allows detecting arbitrary changes in the original (high-dimensional) distribution as long as they affect the average reconstruction loss of the encoder-decoder. Since the true $\mu_{t^*}$ and $\mu_{t^*+1}$ are unknown, we estimate them from the stream:
\begin{equation}
    \hat{\mu}_{1,t^*} = \frac{1}{t^*} \sum_{i=1}^{t^*} L_i, \quad \hat{\mu}_{t^*+1,t} = \frac{1}{t-t^*} \sum_{i=t^*+1}^{t} L_i.
\end{equation} 

\subsection{Detecting the change point}\label{sec:bernstein}

ABCD detects a change at $t^*$ if  $\hat\mu_{1,t^*}$ differs \emph{significantly} from $\hat\mu_{t^*+1,t}$. To quantify this, we derive a test based on Bernstein's inequality~\citep{bernstein_modification_1924}. It is often tighter than more general alternatives like Hoeffding's inequality~\citep{boucheron_concentration_2013}. Let $\hat \mu_1, \hat \mu_2$ be the averages of two independent samples from two univariate random variables. 
One wants to evaluate if both random variables have the same expected values: The null hypothesis $H_0$ is $\mu_1 = \mu_2$.
Based on the two samples, one rejects $H_0$ if $\Pr\left(\lvert\hat \mu_1 - \hat \mu_2\rvert\ \geq \epsilon\right) \leq \delta$ where $\delta$ is a preset significance level.
The following theorem allows evaluating \Cref{eq:concept-change-simplified} based on Bernstein's inequality. 

\begin{theorem}[Bound on $\Pr\left(\lvert\hat{\mu}_1 - \hat{\mu}_2\rvert \geq \epsilon\right)$]\label{theorem:change-score} 
Given two independent samples $X_1,X_2$ of size $n_1$ and $n_2$ from two random variables with unknown expected values $\mu_1, \mu_2$ and variances $\sigma^2_1, \sigma^2_2$. Let $\hat{\mu}_1, \hat{\mu}_2$ denote the sample means and let $\lvert\mu_1-x_i\rvert < M$ for all $x_i\in X_1$ and $\lvert\mu_2-x_i\rvert < M$ for all $x_i\in X_2$ respectively. 
Assuming $\mu_1 = \mu_2$, we have:
\begin{multline}
    \label{eq:bernstein-p}
    \Pr\left(\lvert\hat{\mu}_1 - \hat{\mu}_2\rvert \geq \epsilon\right) \leq \\ 2\exp\left\{-\frac{n_1(\kappa\epsilon)^2}{2\left(\sigma_1^2 + \frac{1}{3}\kappa M \epsilon\right)}\right\} + 2\exp\left\{-\frac{n_2((1-\kappa)\epsilon)^2}{2\left(\sigma_2^2 + \frac{1}{3}(1-\kappa) M \epsilon\right)}\right\} \in (0,4] \\ \forall\kappa \in [0, 1].
\end{multline}
\end{theorem} 

\begin{proof} 
We follow the same steps as in~\citep{bifet_learning_2007, pears_detecting_2014}.

Recall Bernstein's inequality: Let $x_1, \ldots, x_n$ be independent random variables with sample mean $\hat \mu = 1/n\sum x_i$ and expected value $\mu$ s.th. $\forall x_i: \lvert x_i - \mu\rvert \leq M$. Then, for all $\epsilon > 0$,
\begin{equation}
    \Pr\left(\lvert\hat{\mu} - \mu\rvert \geq \epsilon\right) \leq 2\exp\left\{-\frac{n\epsilon^2}{2\left(\sigma^2 + \frac{1}{3} M \epsilon\right)}\right\}.
\end{equation}
We apply the union bound to $\Pr\left(\lvert\hat{\mu}_1 - \hat{\mu}_2\rvert \geq \epsilon\right)$. For all $\kappa \in [0,1]$, we have:
\begin{multline}
\label{eq:union-bound}
    \Pr\left(\lvert\hat{\mu}_1 - \hat{\mu}_2\rvert \geq \epsilon\right) \leq \Pr\left(\lvert\hat{\mu}_1 - \mu_1 \rvert \geq \kappa \epsilon\right) + \Pr\left(\lvert\hat{\mu}_2 - \mu_2 \rvert \geq (1-\kappa) \epsilon\right) 
\end{multline}
Substituting above with Bernstein's inequality completes the proof. 
\end{proof}

With regard to change detection, one can use \Cref{eq:bernstein-p} to evaluate for a time point $k$ if a change occurred. 
The question is, however, how to choose $\epsilon$ to limit the probability of false alarm at any time $t$ to a maximum $\delta$.

Our approach is to set $\epsilon$ to the observed $\lvert\hat{\mu}_{1,k} - \hat{\mu}_{k+1,t}\rvert$ and to set $n_1 = k$, $n_2 = t-k$. The result bounds the probability of observing $\lvert\hat{\mu}_{1,k} - \hat{\mu}_{k+1,t}\rvert$ between two independent samples of sizes $k$ and $t-k$ under $H_0$. 
If this probability is very low, the distributions must have changed at $k$. 
Then, we search for changes at multiple time points $k$ in the current window. Hence, we obtain multiple such probability estimates; our change score is their minimum:
\begin{equation}
\label{eq:change-score} 
p = \min_k \Pr\left(\lvert\hat \mu_1 - \hat \mu_2\rvert \geq \lvert \hat \mu_{1,k} - \hat \mu_{k+1,t} \rvert\right)
\end{equation} 
The corresponding change point $t^*$ splits $(L_1, L_2, \ldots, L_t)$ into the two subwindows with the statistically most different mean. 

\subsubsection{Choice of parameter $\kappa$}

The bound in \Cref{eq:bernstein-p} holds for any $\kappa \in [0,1]$. A good choice, however, provides a tighter estimate, resulting in faster change detection for a given rate of allowed false alarms $\delta$. \citep{bifet_learning_2007} suggest to choose $\kappa$ s.th. $Pr(\lvert \hat{\mu}_1 - \mu_1 \rvert \geq \kappa\epsilon) \approx Pr(\lvert \hat{\mu}_2 - \mu_2\rvert \geq (1- \kappa)\epsilon)$, that approximately minimizes the upper bound. Substituting both sides with Bernstein's inequality, we get
\begin{equation}
\frac{n_1 (\kappa\epsilon)^2}{\sigma_1^2 + \frac{ \kappa M \epsilon}{3}} = \frac{n_2 (1-\kappa)^2\epsilon^2}{\sigma_2^2 + \frac{ (1-\kappa) M \epsilon}{3} }.
\end{equation} 
Setting $n_1 = rn_2$ and simplifying, we have
\begin{equation}
\label{eq:kappa-intermediate}
    \frac{3 \sigma_1^2 +  \kappa M \epsilon}{r\kappa^2} = \frac{3\sigma_2^2 +  (1-\kappa) M \epsilon}{(1-\kappa)^2}.
\end{equation}

To solve for $\kappa$, note that $\lvert\hat{\mu}_{1,k} - \hat{\mu}_{k+1,t}\rvert \approx 0$ for large enough $k$ and $t-k$ while there is no change. This leads to a change score $p \gg \delta$ for any choice of $\kappa$. 
Hence, choosing $\kappa$ optimal is irrelevant while there is no change.

In contrast, if a change occurs, the change in the model's loss dominates the variance in both subwindows, leading to $ M\epsilon \gg \sigma_1^2, \sigma_2^2$. In that case, the influence of $\sigma_1^2, \sigma_2^2$ is negligible for sufficiently large $\kappa$ and $1-\kappa$: 
\begin{equation}
    \label{eq:kappa-intermediate2}
    \frac{\kappa M \epsilon}{r\kappa^2} = \frac{(1-\kappa) M \epsilon}{(1-\kappa)^2}.
\end{equation}
Solving \Cref{eq:kappa-intermediate2} for $\kappa$ results in our recommendation for $\kappa$ (\Cref{eq:kappa-final} which we restrict to $[\kappa_{min}, 1-\kappa_{min}]$ with $\kappa_{min}=0.05$.
\begin{equation}
\label{eq:kappa-final}
    \kappa = \frac{1}{1 + r} = \frac{n_2}{n_1 + n_2}
\end{equation} 

\subsubsection{Minimum sample sizes and outlier sensitivity}\label{sec:outlier-sensitivity}

This section investigates the conditions under which \ac{name} detects changes. 

We derive a minimum size of the first window above which ABCD detects a change. It bases on the fact that the number of observations before an evaluated time point $k$ remains fixed while the number of observations after $k$ grows with $t$. Those counts are $n_1=k$ and $n_2=t-k$ in \Cref{eq:bernstein-p}. Also, since we consider bounded random variables, their variance is bounded as well. Hence, the second term in \Cref{eq:bernstein-p} approaches 0 for any $\epsilon > 0$.
With this, solving \Cref{eq:bernstein-p} for $n_1$ yields:
\begin{equation}
    n_1 \geq \left\lceil2\log\left(\frac{2}{\delta}\right)\left(\frac{\sigma_1^2}{(\kappa\epsilon)^2} + \frac{M}{3\kappa\epsilon}\right)\right\rceil .
\end{equation}

By setting $\epsilon = \lvert \hat \mu_1 - \hat  \mu_2 \rvert$ we see that the required size of the first window decreases the larger the change in the average reconstruction error. For example, with $M=1$, $\epsilon = \sigma_1 = 0.1$, and $\delta = 0.05$ our approach requires $n_1 \geq 32$.

\reviewerB
Since ABCD detects changes in the average reconstruction loss of a bounded vector, it is stable with respect to outliers as long as they are reasonably rare. To see this, assume w.l.o.g. that window~1 contains $n_{out}$ outliers and that $\epsilon>0$. One can show that the average of the outliers, $\hat\mu_{out}$, must exceed the average of the remaining inliers, $\hat\mu_{in}$, by $n_1\epsilon/n_{out}$. In the example above, a single outlier would thus have to exceed $\hat\mu_{in}$ by $n_1\epsilon=3.2$. This, however, is impossible because $M=1$ bounds the reconstruction loss.
\erevision

\subsection{Detecting the change subspace}\label{sec:subspace}

After detecting a change, we identify the change subspace. 
Restricting the encoding size to $d' < d$ forces the model to learn relationships between different input dimensions. As a result, the loss observed for dimension $j$ contains not only information about the change in that dimension (i.e., the marginal distribution in $j$ changes), but also about correlations influencing dimension $j$. Hence, we can detect changes in the marginal- and joint-distributions by evaluating in which dimensions the loss changed the most. 

\Cref{alg:subspace} describes how we identify change subspaces. For each dimension $j$, we compute the average reconstruction loss (the squared error in dimension $j$) before and after $t^*$, denoted $\hat{\mu}^j_{1,t^*},\hat{\mu}^j_{t^*+1,t}$ (lines~5 and 6), and the standard deviation $\sigma^j_{1,t^*}, \sigma^j_{t^*+1,t}$ (lines~6 and 7). We then evaluate \Cref{eq:bernstein-p}, returning an upper bound on the $p$-value in the range $(0,4]$ for dimension $j$ (line~9). If $p_j < \tau \in [0,4]$, an external parameter for which we give a recommendation later on, we add $j$ to the change subspace (lines~10 and 11). 

\begin{algorithm}[htb]
\begin{algorithmic}[1]
\caption{Identification of change subspaces.}
\label{alg:subspace}
\Require{$(x_1, \hat x_1), \ldots, (x_t, \hat x_t)$, $t^*$}
\Procedure{\textsc{Subspace}}{}
\State $D^*\gets \emptyset$
\ForAll{$j \in 1, \ldots, d$}
\State $s \gets \left ( (x_{i,j}-\hat x_{i,j})^2 \ \forall i \in 1, \ldots, t \right)$ 
\State $\hat \mu^j_{1,t^*} = \frac{1}{t^*}\sum_{i=1}^{t^*} s_i, \quad \hat \mu^j_{t^*+1,t} = \frac{1}{t-t^*}\sum_{i=t^*+1}^{t} s_i$
\State $\sigma^j_{1,t^*} = \sqrt{ \frac{1}{t^*} \sum_{i=1}^{t^*} \left(s_i - \hat \mu^j_{1,t^*}\right)^2 }$
\State $\sigma^j_{t^*+1,t} = \sqrt{ \frac{1}{t-t^*} \sum_{i=t^*+1}^{t} \left(s_i - \hat \mu^j_{t^*+1,t}\right)^2 }$
\State $p_j \gets$ Evaluate \Cref{eq:bernstein-p} \Comment{Bernstein score}
\If{$p_j < \tau$}
\State $D^* \gets D^* \cup \left\{j\right\}$
\EndIf
\EndFor
\State \textbf{Return} $D^*$
\EndProcedure
\end{algorithmic}
\end{algorithm}

\subsection{Quantifying change severity}\label{sec:severity}

ABCD provides a measure of change severity in the affected subspace, based on the assumption that the loss in the change subspace increases with severity. Hence, we compute the average reconstruction loss observed in $D^*$ before and after the change, 
\begin{equation}
    \hat{\mu}_{1, t^*}^{D^*} = \frac{1}{\lvert D^*\rvert t^*} \sum_{i=1}^{t^*}\sum_{j\in D^*} L_{i,j}, \quad \hat{\mu}_{t^*+1, t}^{D^*} = \frac{1}{\lvert D^*\rvert (t-t^*)} \sum_{i=t^*+1}^{t}\sum_{j\in D^*} L_{i,j}
\end{equation}
and the standard deviation observed before the change:
\begin{equation}
    \sigma_{1,t^*}^{D^*} = \sqrt{\frac{1}{t^*} \sum_{i=1}^{t^*} \left(\hat\mu_i^{D^*} - \hat\mu^{D^*}_{1,t^*}\right)^2} \text{ with } \hat\mu_i^{D^*} = \frac{1}{\lvert D^* \rvert} \sum_{j\in D^*} L_{i,j}
\end{equation}
We then standard-normalize the average reconstruction loss $\hat{\mu}_{t^*+1}^{D^*}$ observed after the change:
\begin{equation}
\label{eq:severity}
    \Delta = \frac{\left\lvert\hat{\mu}_{t^*+1, t}^{D^*} - \hat{\mu}^{D^*}_{1,t^*}\right\rvert}{\sigma_{1,t^*}^{D^*}} \in \mathbb{R}^+
\end{equation} 
Intuitively, $\Delta$ is the standard deviation of model's loss on the new distribution. 

\subsection{Working with windows}\label{sec:stats}

In comparison to most approaches, \ac{name} evaluates multiple possible change points within an adaptive time interval $[1, \ldots, t]$. This frees the user from choosing the window size a-priori and allows to detect changes at variable time scales. Next, we discuss how to efficiently evaluate those time points.

\subsubsection{Maintaining loss statistics online}

To avoid recomputing average reconstruction loss values and their variance for multiple time points every time new observations arrive, we store Welford aggregates $A_{1,k}$ summarizing the stream in the interval $[1, \ldots, k]$. Each aggregate $A_{1,k}$ is a tuple containing the average reconstruction loss $\hat \mu_{1,k}$ and the sum of squared differences $ssd_{1,k} = k^{-1} \sum_{j = 1}^k L_j$.
We store these aggregates for the time interval $[1, \ldots, t]$.

\textbf{Creating a new aggregate.} Every time a new observation with loss $L_{t}$ arrives, we create a new aggregate based on the previous aggregate $A_{1,t-1} = (\hat{\mu}_{1,t-1}, ssd_{1,t-1})$ in $\mathcal{O}(1)$ using Welford's algorithm~\citep{knuth_art_1997}:
    \begin{equation}
    \label{eq:welford-1}
        \hat{\mu}_{1,t} = \hat{\mu}_{1,t - 1} + \frac{1}{t}(L_t - \hat{\mu}_{1,t-1})
    \end{equation}
    \begin{equation}
    \label{eq:welford-2}
        ssd_{1,t} = ssd_{1,t-1} + \left(L_{t} - \hat{\mu}_{1,t-1}\right)\left(L_{t}-\hat{\mu}_{1,t}\right)
    \end{equation}

\textbf{Computing the statistics.}
Two aggregates $A_{1,k}$ and $A_{1,t}$, $t>k$ overlap in the time interval $[1, \ldots, k]$. We leverage this overlap to derive an aggregate $A_{k+1,t} = (\hat \mu_{k+1,t}, ssd_{k+1,t})$ representing the time interval $[k+1, \ldots, t]$. \Cref{eq:mu-latter-sw} and \Cref{eq:ssd-latter-sw} are based on Chan's method for combining variance estimates of non-overlapping samples~\citep{chan_updating_1982}.
\begin{equation}
    \label{eq:mu-latter-sw}
    \hat{\mu}_{k+1,t} = \frac{1}{t - k}(t\hat{\mu}_{1,t}-k\hat{\mu}_{1,k})
\end{equation}
\begin{equation}
    \label{eq:ssd-latter-sw}
    ssd_{k+1,t} = ssd_{1,t} - ssd_{1,k} - \frac{k(t-k)}{t}\left(\hat{\mu}_{1,k} - \hat{\mu}_{k+1,t}\right)^2
\end{equation}

From $ssd_{1,k}$ and $ssd_{k+1,t}$ we can compute the sample variances as follows:
\begin{equation}
    \sigma^2_{1,k} = \frac{ssd_{1,k}}{k - 1},\quad \sigma^2_{k+1,t} = \frac{ssd_{k+1,t}}{t-k-1}
\end{equation}

\begin{derivation*} 
Given two non-overlapping samples $A = \{x_1, \ldots, x_m\}$ and $B= \{x_1, \ldots, x_n\}$ of a real random variable. Let $T_A = \sum_{i=1}^m x_i$ and $T_B = \sum_{i=1}^n x_i$ be the sums of the samples and $ssd_A = \sum_{i=1}^m (x_i - m^{-1}T_A)^2$ and $ssd_B = \sum_{i=1}^n (x_i - n^{-1}T_B)^2$ be the sums of squared distances from the mean.

For the union of both sets $AB = A \cup B$ we have $T_{AB} = T_A + T_B$, which is equivalent to $(m+n)\hat \mu_{AB} = m\hat \mu_A + n\hat \mu_B$. Solving for $\hat \mu_B$ gives
\begin{equation}
    \hat\mu_B = \frac{m+n}{n}\hat\mu_{AB} - \frac{m}{n} \hat\mu_A.
\end{equation}

Substituting $n=t-k$, $m=k$, $\hat\mu_A = \hat\mu_{1,k}$, $\hat\mu_B = \hat\mu_{k+1,t}$, and $\hat\mu_{1,t} = \hat\mu_{AB}$ gives \Cref{eq:mu-latter-sw}; next we derive \Cref{eq:ssd-latter-sw}. \cite{chan_updating_1982} state: 
\begin{equation}
    ssd_{AB} = ssd_A + ssd_B + \frac{m}{n(m+n)}\left(\frac{n}{m}T_A - T_B\right)^2,
\end{equation}
which is equivalent to
\begin{equation}
    ssd_{AB} = ssd_A + ssd_B + \frac{m}{n(m+n)}\Bigg(n\underbrace{\left(\frac{1}{m}T_A - \frac{1}{n}T_B\right)}_{=\hat\mu_A-\hat\mu_B}\Bigg)^2.
\end{equation} 

Solving for $ssd_B$, applying the former substitutions, and setting $ssd_A = ssd_{1,k}$, $ssd_B = ssd_{k+1,t}$, and $ssd_{1,t} = ssd_{AB}$ results in \Cref{eq:ssd-latter-sw}.
\end{derivation*}

\subsection{Implementation}\label{sec:implementation}

\paragraph{Algorithm}
One can implement \ac{name} as a recursive algorithm, see \Cref{alg:overview}, which restarts every time a change occurs. 
We keep a data structure $W$ that contains the aggregates, instances, and reconstructions. $W$ can either be empty, or, in the case of a recursive execution, already contain data from the previous run. 


Prior to execution, our algorithm must first obtain a model of the current data from an initial sample of size $n_{min}$. If necessary, ABCD allows enough instances to arrive (lines 5--7). 
Larger choices of $n_{min}$ allow for better approximations of the current distribution but delay change detection. Hence our recommendation is to set $n_{min}$ as small as possible to still learn the current distribution; a default of $n_{min} = 100$ has worked well for us. 

Afterwards, the algorithm trains the model using the instances in $W$ (lines 8--9). ABCD can in principle work with various encoder-decoder models; thus we deal with tuning the model only on a high level. Nonetheless, we give recommendations in our sensitivity study later on. 

After model training, \ac{name} detects changes. It reconstructs each new observation $x_{t+1}$ (line 11), creates a new aggregate $A_{1,t+1}$ (line 12), and adds $w_{t+1} \coloneqq (A_{1,t+1}, \hat{x}_{t+1}, x_{t+1})$ to $W$ (lines 13--14). Our approach then computes change score $p$ and change point $t^*$ (lines 15--16). If $p < \delta$, it detects a change. 

Once \ac{name} detects a change, it identifies the corresponding subspace and evaluates its severity (lines 21--22). Then it adapts $W$ by dropping the outdated part of the window (line 23), including all information obtained with the outdated model. 
At last, we restart ABCD with the adapted window (line 24).

\begin{algorithm}[htb]
\begin{algorithmic}[1]
\caption{\acf{name}}
\label{alg:overview}
\Require{The model $\psi\circ\phi$, threshold $\delta$, threshold $\tau$}
\Procedure{\textsc{ABCD}}{$W$}
\State $\psi\circ\phi \gets Null$; $t \gets \lvert W \rvert$ \Comment{Model not yet trained}
\While{new instance $x_{t+1}$} 
\If{$t < n_{min}$}\Comment{Warm up}
    \State $w_{t+1} \gets (-, -, x_{t+1})$
    \State $W \gets W\ \lvert\rvert\ w_{t+1}$ 
\ElsIf{$\psi\circ\phi = Null$}
    \State $\psi\circ\phi \gets \textsc{TrainModel(}W\textsc{)}$
\Else
\State $\hat x_{t+1} \gets \psi(\phi(x_{t+1}))$ \Comment{Reconstruct}
\State $A_{t+1} \gets$ update aggregate $A_t$ with $L_{t+1}$
\State $W \gets W\ \lvert\rvert\ (A_{t+1}, \hat x_{t+1}, x_{t+1})$
\State $p\gets$ \Cref{eq:change-score} \Comment{Bernstein score}
\State $t^* \gets \argmin_k$ of \Cref{eq:change-score}
\If{$p < \delta$} \Comment{A change occurred}
    \State $D^* \gets \textsc{Subspace}(W, t^*, \tau)$
    \State $\Delta \gets \textsc{Severity}(W, t^*, D^*)$
    \State $W \gets \{ (-, -, x_i)\ \forall w_i \in W: i > t^* \}$
    \State ABCD$(W)$ \Comment{Restart}
\EndIf

\EndIf
\EndWhile
\EndProcedure
\end{algorithmic}
\end{algorithm}

\paragraph{Discussion}

In the worst case our approach consumes linear time and memory because $W$ grows linearly with $t$. However, we can simply restrict the size of $W$ to $n_{max}$ items for constant memory or evaluate only $k_{max}$ window splits for constant runtime. In the latter case we split $W$ at every $t/k_{max}$th time point. 
Regarding $n_{max}$, it is beneficial that the remaining aggregates still contain information about all observations in $(1, \ldots, t)$. Hence, \ac{name} considers the \emph{entire} past since the last change even though one restricts the size of $W$.

\reviewerA
ABCD can work with any encoder-decoder model, such as deep neural networks. However, handling a high influx of new observations faster than the model's processing capability can be challenging.
\reviewerA
Assuming that $\psi\circ\phi\in\mathcal{O}(g(d))$ for some function $g$ of dimensionality $d$, the processing time of a single instance during serial execution is in $\mathcal{O}\left(g(d) + k_{max}\right)$.
Nevertheless, both the deep architecture components and the computation of the change score (cf. Equation \ref{eq:change-score}) can be executed in parallel using specialized hardware. 
\erevision

\reviewerB
Dimensionality reduction techniques are often already present in data stream mining pipelines, for example as a preprocessing step to improve the accuracy of a classifier~\citep{yan_effective_2006}. Reusing an existing dimensionality reduction model makes it is easy to integrate ABCD into an existing pipeline. 
\erevision

Bernstein's inequality holds for zero-centered bounded random variables that take absolute values of at maximum $M$ almost surely. While $M=1$ serves as a theoretical upper limit of the zero-centered reconstruction error $L_t - \mathbb{E}[L_t]$ for $x_t\in [0,1]^d$, we observe that this theoretical limit is very conservative in practice (cf.~\Cref{app:reconstruction-loss}).
In fact, observing an error of $1$ corresponds to an instance and reconstruction of $x = [0]^d$ and $\hat x = [1]^d$. This leads us to use $M=0.1$ in our experiments.

\section{Experiments} \label{sec:experiments}

This section describes our experiments and results. We first describe the experimental setting (\Cref{exp:setting}). Then we analyze ABCD's change detection performance (\Cref{exp:change-point}), its ability to find change subspaces and quantify change severity (\Cref{exp:subspace-severity}), and its parameter sensitivity (\Cref{exp:parameters}).

\subsection{Algorithms}\label{exp:setting}

We evaluate \ac{name} with different encoder-decoder models: (1) \ac{pca} ($d'=\eta d$), (2) Kernel-PCA ($d' = \eta d$, RBF-kernel), 
\reviewerA and (3) a standard fully-connected autoencoder model with one hidden ReLU layer ($d'=\eta d$) and an output layer with sigmoid activation. For (1) and (2), we rely on the default scikit-learn implementations. We implement the autoencoder (3) in pytorch and train it through gradient descent using $E$ epochs and an Adam optimizer with default parameters according to~\cite{kingma_adam_2015}; see \Cref{app:ae-training} for pseudocode of the autoencoder training procedure. 
\erevision

We compare ABCD with AdwinK, IKS, IBDD, WATCH, and D3 (c.f. \Cref{sec:relatedwork}). We evaluate for each approach a large grid of parameters, shown in \Cref{tab:hyperparameters}. 
\reviewerA
Whenever possible, the evaluated grids of hyperparameters for competitors base on recommendations in respective papers. Otherwise, we choose them based on preliminary experiments. For ABCD, we evaluate larger and smaller values for $\delta$, $\eta$ and $E$ to observe our approach's sensitivity to those parameters. 
The choice of $\tau=2.5$ is our recommended default based on our sensitivity study in \Cref{exp:parameters}. Last, we set $n_{min}=100$ and $k_{max}=20$, minimum values that have worked well in preliminary experiments. 
\erevision

\begin{table}[tb]
    \centering
    \caption{Evaluated approaches and their parameters.}
    \begin{threeparttable}
    \begin{tabular}{lll}
    \toprule
        Algorithm& Parameter& Values\\
        \midrule
         ABCD &model& PCA, Kernel PCA, Autoencoders \\
         & $\delta$ & $0.05$\\
         & $\eta$ & $0,3, 0.5, 0.7$ \\
         & $E$\tnote{\dag}
         & $20, 50, 100$ \\
         & $n_{min};k_{max};\tau$ & 100; 20; 2.5 \\
         \midrule
         AdwinK& $k$ & $0.01\tnote{*}, 0.05\tnote{*}, 0.1\tnote{*}, 0.2\tnote{*}, 0.3\tnote{*}, 0.4\tnote{*}, 0.5\tnote{*}$\\
         &$\delta$& $0.05$\\ 
         \midrule
         D3 & $\omega$ & $100\tnote{*}, 250\tnote{*}, 500\tnote{*}$\\
         & $\rho$ & $0.1\tnote{*}, 0.2\tnote{*}, 0.3\tnote{*}, 0.4\tnote{*}, 0.5\tnote{*}$\\
         & $\tau$ & $0.6\tnote{*}, 0.7\tnote{*}, 0.8\tnote{*}, 0.9\tnote{*}$\\
         &model& Logistic Regression\tnote{*}, Decision Tree\\
         & tree depth & 1, 3, 5\\
         \midrule
         IBDD & $\omega$ & $100, 200, 300$\\
         & $m$& $10, 20, 50, 100$\\
         \midrule
         IKS & $W$ & $100\tnote{*}, 200, 500\tnote{*}$ \\
         & $\delta$ & $0.05$ \\
         \midrule
         WATCH\tnote{\ddag}& $\omega$ & $500, 1000$\\
         &$\kappa$& $100$\\
         & $\epsilon$ & $2, 3$\\
         &$\mu$& $1000, 2000$\\
         \bottomrule
    \end{tabular}
    \begin{tablenotes}
    \item[*]used or recommended in the respective papers
    \item[\dag]only relevant for autoencoders
    \item[\ddag]authors did not recommend parameters for their approach
    \end{tablenotes}
    \end{threeparttable} 
    \label{tab:hyperparameters}
\end{table}

\subsection{Datasets}\label{sec:datasets}

There are not many public benchmark data streams for change detection. 
Thus we generate our own from seven real-world (rw) and synthetic (syn) classification datasets, similar to \citep{faber_watch_2021,faithfull_combining_2019}. We simulate changing data streams\footnote{Available at \href{https://github.com/heymarco/AdaptiveBernsteinChangeDetector}{https://github.com/heymarco/AdaptiveBernsteinChangeDetector}} by sorting the data by label, unless stated otherwise. If the label changes, a change has occurred. 
\reviewerB
In real-world data streams, the number of observations between changes depends on each dataset, reported below. In the synthetic streams, we introduce changes every 2000 observations, which is a relatively large interval, to assess whether some approaches generate many false alarms.
\erevision
The generators base on the following datasets: 
\begin{itemize}
    \item \textbf{HAR (rw):} The dataset \emph{Human Activity Recognition with Smartphones}~\citep{anguita_public_2013} ($d=561$) bases on smartphone accelerometer and gyroscope readings for different actions a person performs. \reviewerB A change occurs on average every 1768 observations. \erevision 
    \item \textbf{GAS (rw):} This data set~\citep{vergara_gas_2011} ($d=128$) contains data from 16 sensors exposed to 6 gases at various concentrations.
    \reviewerB A change occurs on average every 2265 observations. \erevision 
    \item \textbf{LED (syn):} The LED generator samples instances representing a digit on a seven segment display. It contains 17 additional random dimensions. We add changes by varying the probability of bit-flipping in the relevant dimensions.
    \item \textbf{RBF (syn):} The RBF generator \citep{bifet_leveraging_2010} starts by drawing a fixed number of centroids. For each new instance, the generator chooses a centroid at random and adds Gaussian noise. 
    To create changes, we increment the seed of the generator resulting in different centroids. We then use samples from the new generator in a subspace of random size.
    \item \textbf{MNIST, FMNIST, and CIFAR (syn):} Those data generators sample from the image recognition datasets MNIST~\citep{lecun_mnist_2010}, Fashion MNIST (FMNIST)~\citep{xiao_fashion-mnist_2017} ($d=784)$, and CIFAR~\citep{krizhevsky_learning_2009} ($d=1024$, grayscale).  
\end{itemize} 

Changes can occur rapidly (``abrupt'' or ``sudden'') or in time intervals (``gradual'' or ``incremental''). The shorter the interval, the more sudden the change. We vary the interval size between 1 and 300 unless stated otherwise. Real-world and image data do not have a ground truth for change subspaces and severity. Thus we generate three additional data streams: 
\begin{itemize}
    \item \textbf{HSphere (syn):} This generator draws from a $d^*$-dimensional hypersphere bound to $[0,1]$ and adds $d-d^*$ random dimensions. We vary the radius and center of the hypersphere to introduce changes. The change subspace contains those dimensions that define the hypersphere.
    \item \textbf{Normal-M/V (syn):} These generators sample from a $d^*$-di\-men\-sio\-nal normal distribution and add $d-d^*$ random dimensions. For type \textbf{M}, changes affect the distribution's mean, for \textbf{V} we change the distribution's variance.
\end{itemize}

\subsection{Change point detection}\label{exp:change-point}

We use precision, recall, and F1-score to evaluate the performance of the approaches at detecting changes. We define true positives (TP), false positives (FP) and false negatives (FN) as follows:
\begin{itemize}
    \item \textbf{TP:} A change was detected before the next change. 
    \item \textbf{FN:} A change was not detected before the next change.
    \item \textbf{FP:} A change was detected although no change occurred. 
\end{itemize}
Also, we report the mean time until detection (MTD) indicating the average number of instances until a change is detected. 

\begin{figure*}[htb]
    \centering
    \includegraphics[width=\linewidth]{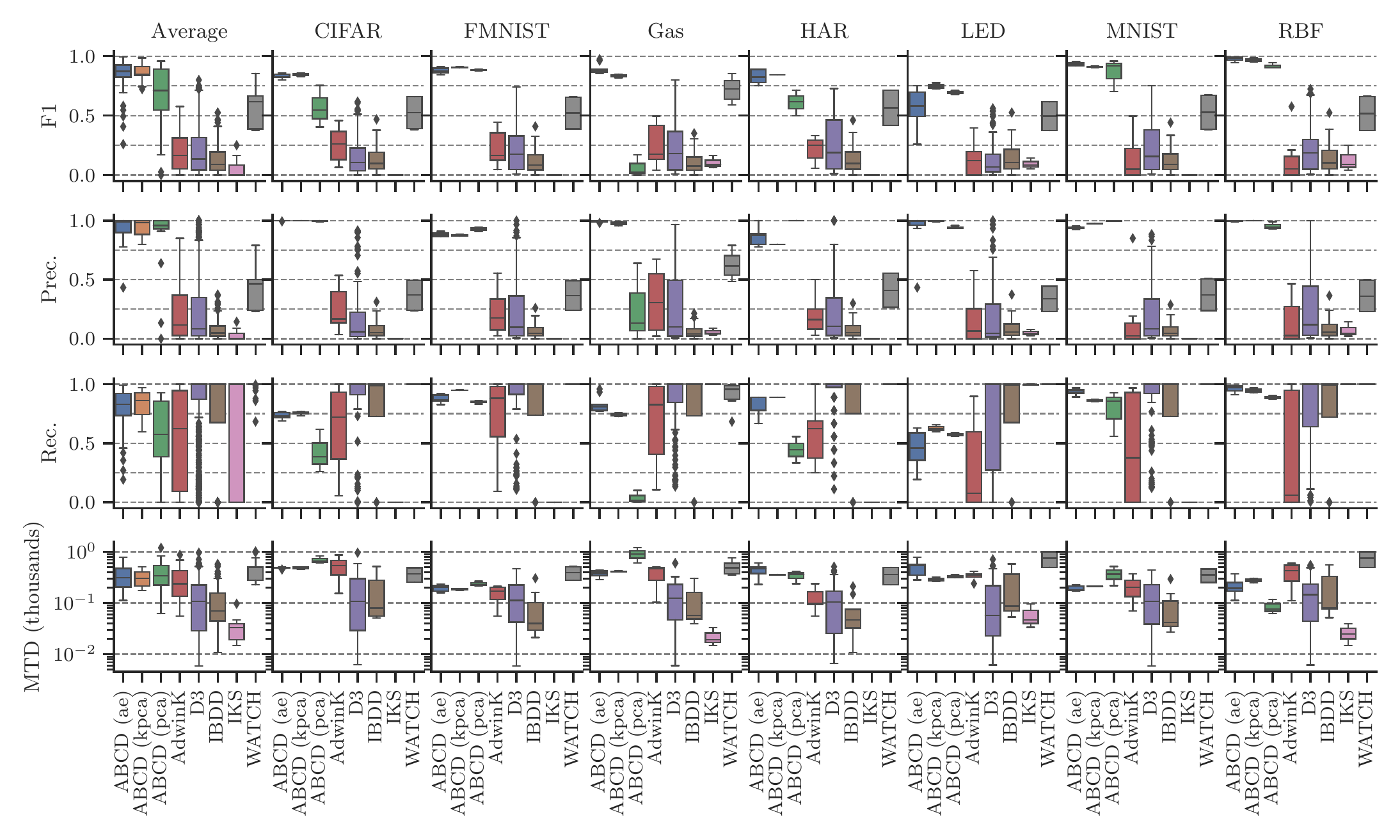}
    \caption{Change Point Detection: Results for different algorithms and datasets; each box contains the results for the evaluated grid of parameters.}
    \label{fig:results-cp}
\end{figure*} 

\Cref{fig:results-cp} shows F1-score, precision, recall, and MTD for all datasets and algorithm, as well as a column ``Average'' that summarizes across datasets. 
Each box contains the results for the grid of hyperparameters shown in \Cref{tab:hyperparameters}. 
We see that our approach outperforms its competitors w.r.t. F1-score and precision. It also is competitive in terms of recall, though it loses against IKS, IBDD, and WATCH. These approaches seem overly sensitive. 
The results also indicate that ABCD works well for a wide range of hyperparameters. 
\reviewerA
One reason is that ABCD uses adaptive windows, thereby eliminating the effect of a window size parameter (demonstrated in Section~\ref{exp:window_types}).
Another reason is that ABCD detects changes in reconstruction loss irrespective of the actual quality of the reconstructions. For instance, Kernel PCA and PCA produce reconstructions of different accuracy in our experiments. However, for both models, the average accuracy changes when the stream changes, which is what our algorithm detects. Refer to \Cref{app:reconstruction-loss} for an illustration of the models' reconstruction loss over time. Hence, our reported results do not yield information about the actual accuracy of the underlying encoder-decoder models. 
\erevision

ABCD has a higher MTD than D3, IBDD, and IKS, i.e., it requires more data to detect changes. 
However, those competitors are much less conservative and detect many more changes than exist in the data. Hence they have low precision but high recall~--- this leads to a lower MTD.  

\Cref{tab:res-summary} reports the results of all approaches with their best hyperparameters. WATCH and D3 achieve relatively high F1-score and precision. In fact, those approaches are our strongest competitors although we still outperform them by at least 3\,\%. Further, WATCH has an MTD of 626, which is more than \ac{name} while D3 and ABCD have a comparable MTD. 

\reviewerB
ABCD has much higher precision than its competitors. We assume this is because ABCD (1) leverages the relationships between dimensions, in comparison to AdwinK, IKS, or IBDD, and (2) learns those relationships more effectively than, say, D3 or WATCH. For example, we observed in our experiments that WATCH was frequently unable to accurately approximate the Wasserstein distance in high-dimensional data. 

ABCD has lower recall than most competitors, partly due to their over-sensitivity. In this regard, our approach might benefit from application-specific encoder-decoder models that leverage structure in the data, such as spacial relationships between the pixels of an image, more effectively.
\erevision

\begin{table}[t]
\centering
\caption{Results of approaches with their best hyperparameter configuration w.r.t. F1 score averaged over all data sets.}
\begin{subtable}[t]{.54\textwidth}
\vspace{0cm}
\begin{tabular}{llrrrr}
\toprule
\xrowht[()]{2.54pt}
Approach       &    F1 &  Prec. &  Rec. &     MTD \\
\midrule
ABCD (ae) &  0.90 &   0.96 &  0.87 & 250\\
ABCD (kpca) &  0.88 &   0.95 &  0.84 &   312\\
ABCD (pca) &  0.73 &   0.93 &  0.65 &   442\\
\bottomrule
\end{tabular}
\end{subtable}
\hfill
\begin{subtable}[t]{.45\textwidth}
\vspace{0cm}
\begin{tabular}{llrrrr}
\toprule
    AdwinK &  0.46 &   0.48 &  0.57 &   400\\
    D3 &  0.70 &   0.63 &  0.82 &   251\\
    IBDD &  0.45 &   0.30 &  0.97 &   396\\
    IKS &  0.08 &   0.04 &  0.43 &    24\\
    WATCH &  0.69 &   0.54 &  1.00 &   626\\
\bottomrule
\end{tabular}
\end{subtable}
\label{tab:res-summary}
\end{table} 

\subsection{Change subspace and severity}\label{exp:subspace-severity}

We now evaluate change subspace identification and change severity estimation.
We set $d=\{24, 100, 500\}$ and vary the change subspace size $d^*$ randomly in $[1, d]$ (except for LED, here the subspace always contains dimensions 1--7). We set the ground truth for the severity to the absolute difference between the parameters that define the concepts, e.g., the hypersphere-radius in HAR before and after the change. 
We report an approach's subspace detection accuracy (SAcc.), where true positives (true negatives) represent those dimensions that were correctly classified as being member (not being member) of the change subspace. We use Spearman's correlation between the detected severity and the ground truth. We also report the F1-score for detecting change points.

\begin{figure}[htb]
    \centering
    \includegraphics[width=\linewidth]{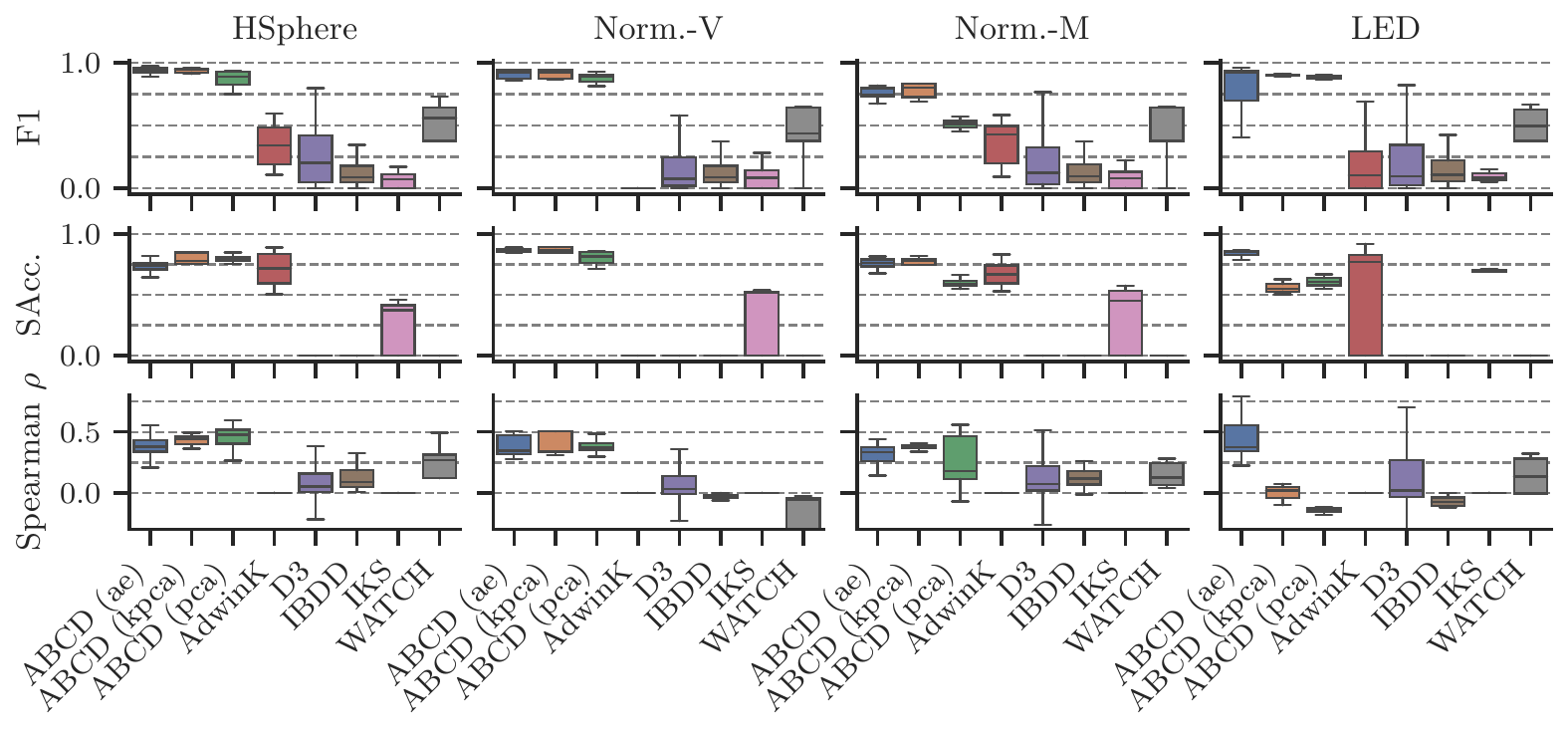}
    \caption{Results for evaluating change subspace and severity.}
    \label{fig:evaluation-region-severity}
\end{figure}

\Cref{fig:evaluation-region-severity} shows our results. As before, each box summarizes the results for the grid of evaluated hyperparameters. 
\reviewerB
Comparing the two approaches, AdwinK and IKS, that monitor each dimension separately, we see that the former can only detect changes that affect the mean of the marginal distributions (i.e., on Norm-M, LED). At the same time, the latter can also detect other changes (e.g., changes in variance). This is expected since AdwinK compares the mean in two windows while IKS compares the empirical distributions.
\erevision

Regarding subspace detection, our approach achieves an accuracy of 0.72 for PCA, 0.78 for autoencoders, and 0.79 for Kernel PCA. AdwinK performs similarly well when changes affect the mean of the marginal distributions. 
\reviewerB
Except on LED, IKS performs worse than ABCD and AdwinK, presumbably because IKS issues an alarm as soon as a single dimension changed. 
\erevision

\reviewerB
The estimates of our approach correlate more strongly with the ground truth than those of competitors, with an average of 0.31 for PCA, 0.36 for Kernel PCA and 0.37 for Autoencoders. However, we expect more specialized models to better than our tested models. 
On LED, PCA-based models appear to struggle to separate patterns from noise, resulting in poor noise level estimates and low correlation scores.
\erevision

\subsection{Parameter sensitivity of \ac{name}} \label{exp:parameters}

\paragraph{Sensitivity to $\eta$} \Cref{fig:eta} plots F1 for different datasets over $\eta$. We observe that the size of the bottleneck does not significantly impact the change detection performance of \ac{name}~(ae) and \ac{name}~(kpca). For PCA, however, too large bottlenecks seem to inhibit change detection on CIFAR, Gas, and MNIST. For those datasets, we assume that the change occurs along the retained main components, rendering it undetectable; see \Cref{app:detectable-undetectable} for an illustration. \Cref{fig:eta-subspace-severity} shows the subspace detection accuracy and Spearman's $\rho$. The influence of $\eta$ on both metrics is low. 
\reviewerB
As mentioned earlier, we assume that a \emph{change} in reconstruction loss, rather than the quality of reconstruction itself, is crucial for ABCD.
An exception is the LED dataset, on which PCA and Kernel-PCA are unable to provide a measure that positively correlates with change severity. 
We hypothesize that those methods struggle to separate patterns from noise, resulting in poor noise level estimates and low correlation scores. 
\erevision

\paragraph{Sensitivity to $E$}
\Cref{fig:E} plots our approach's performance for different choices of $E$. Overall, our approach seems to be robust to the choice of $E$. On LED, however, larger choices of $E$ lead to substantial improvements in F1-score. 
\reviewerB
The reason may be that the autoencoder does not converge to a proper representation of the data for small $E$. 
\erevision
To avoid this, we recommend choosing $E \geq 50$ and to increase the value if one observes that the model has not yet converged sufficiently. 

\paragraph{Sensitivity to $\tau$}
\Cref{fig:tau-sensitivity} investigates how the choice of $\tau$ affects the performance of \ac{name} at detecting subspaces. Since the change score in \Cref{eq:bernstein-p} provides an upper bound on the probability that a change occurred, the function can return values greater than 1, i.e,. in the range $(0,4]$. 
Hence we vary $\tau$ in that range and record the obtained subspace detection accuracy. For all approaches we achieve optimal accuracy at $\tau \approx 2.5$.
\reviewerB
This is probably because some dimensions could change more severely than others, resulting in variations of the change scores observed in the different dimensions of the change subspace. 
\erevision
Based on our findings we recommend $\tau=2.5$ as default.

\begin{figure}[tb]
     \centering
     \begin{subfigure}[t]{\linewidth}
         \centering         \includegraphics[width=\textwidth]{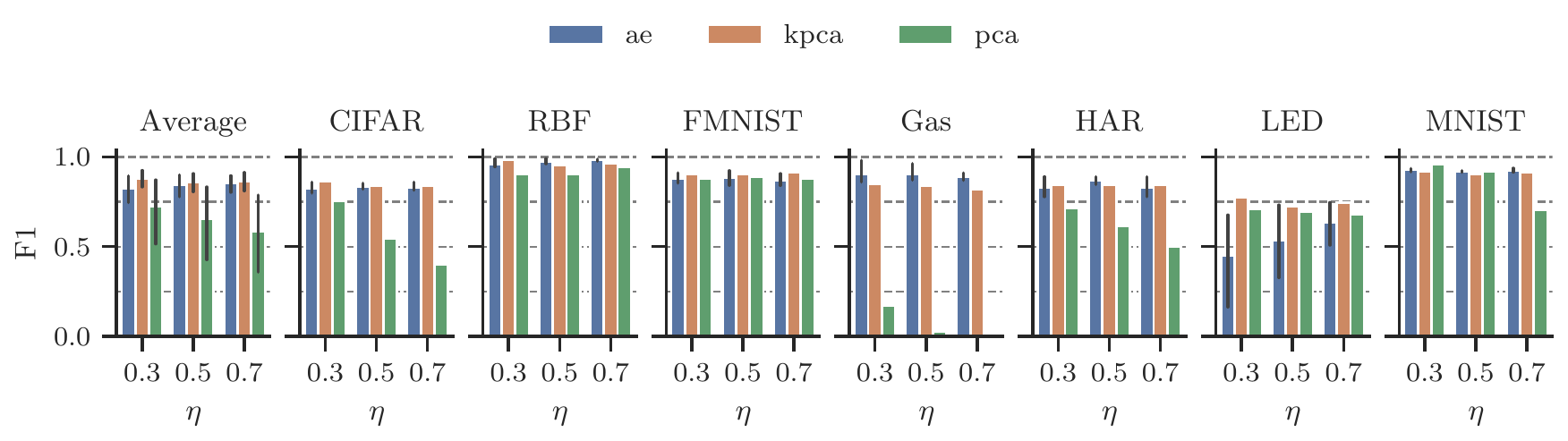}
         \caption{Influence of $\eta$ on the identification of changes.}
         \label{fig:eta}
     \end{subfigure}
          \begin{subfigure}[t]{\linewidth}
         \centering
         \includegraphics[width=\textwidth]{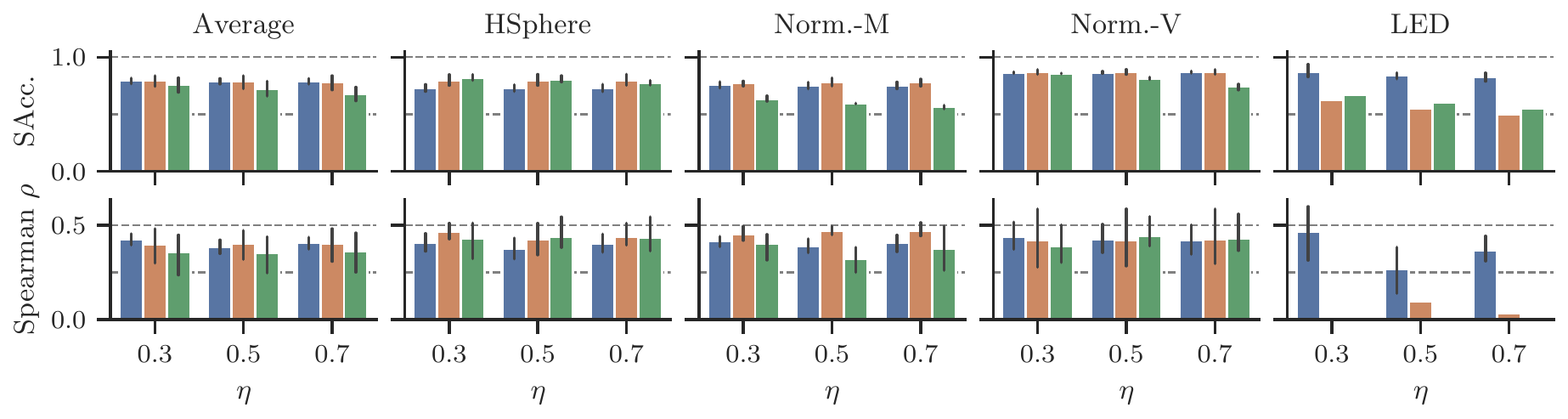}
         \caption{Influence of $\eta$ on the estimation of change subspaces and severity.}
         \label{fig:eta-subspace-severity}
     \end{subfigure}
     \begin{subfigure}[t]{\linewidth}
        \centering
        \includegraphics[width=.9\linewidth]{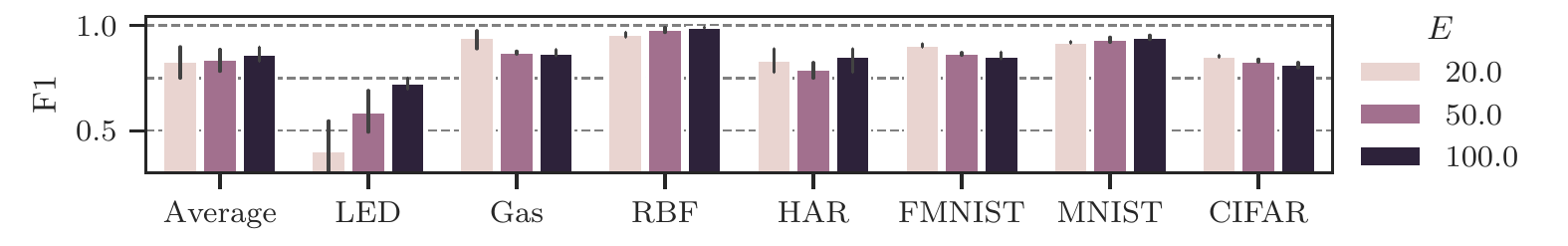}
        \caption{Change detection performance of \ac{name} (ae) depending on $E$.}
        \label{fig:E}
     \end{subfigure}
     \begin{subfigure}[t]{\linewidth}
        \centering
        \includegraphics[width=.85\linewidth]{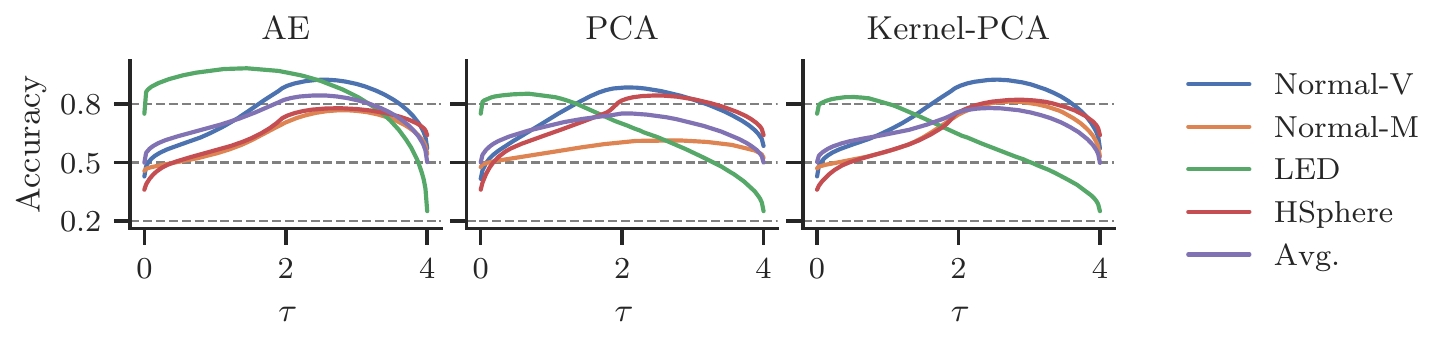}
        \caption{Subspace detection accuracy of \ac{name} depending on $\tau$.}
        \label{fig:tau-sensitivity}
     \end{subfigure}
    \caption{Sensitivity of our approach to its hyperparameters.} 
    \label{fig:cp-sensitivity}
\end{figure}

\subsection{Ablation study on window types}\label{exp:window_types}

Next, we investigate the effect of different window types on change detection performance. We evaluate those commonly found in change detection literature (and in our competitors) and couple them with encoder-decoder models and the probability bound in \Cref{eq:bernstein-p}. In particular, we compare: (1) Adaptive windows (AW), as in ADWIN, AdwinK, and our approach, (2) fixed reference windows (RW), as in IKS, (3) sliding windows (SW), as in WATCH, and (4) jumping windows (JW), as in D3. The latter ``jump'' every $\rho \lvert W \rvert$ instances.

We evaluate the hyperparameters mentioned in \Cref{tab:hyperparameters}. 
For example, because D3 uses jumping windows, we include the evaluated hyperparameters for D3 in our evaluation of jumping windows. 
In addition, we extend the grid with other reasonable choices since we already preselected those in \Cref{tab:hyperparameters} for our competitors in a preliminary study. 
For ABCD we use $\eta=0.5$ and $E=50$.

\Cref{tab:window-ablation} reports the average over all hyperparameter combinations. AWs yield higher F1-score and recall than other techniques, while precision remains high ($\geq 0.95$). SWs have a lower MTD than AWs and hence seem to require a fewer instances until they detect a change. \reviewerB This is expected: in contrast to sliding windows, adaptive windows allow the detection of even slight changes after a longer period of time, resulting in both higher MTD and recall. \erevision 

\begin{table}[t]
    \centering
    \caption{Ablation: Using encoder-decoder models with different window types.}
\begin{tabular}{llrrrr}
\toprule
Model & Window &    F1 &  Prec. &  Rec. &    MTD \\
\midrule
AE & AW &  \textbf{0.83} &   0.95 &  \textbf{0.78} &  455.6 \\
    & RW &  0.53 &   \textbf{1.00} &  0.21 &  403.6 \\
    & SW &  0.62 &   \textbf{1.00} &  0.40 &  \textbf{207.2} \\
        & JW &  0.52 &   0.79 &  0.46 &  239.1 \\
\midrule
KPCA & AW &  \textbf{0.83} &   0.99 &  \textbf{0.75} &  309.0 \\
    & RW &  0.56 &   \textbf{1.00} &  0.23 &  456.3 \\
    & SW &  0.68 &   \textbf{1.00} &  0.49 &  \textbf{202.8} \\
    & JW &  0.50 &   0.77 &  0.33 &  266.2 \\
\midrule
PCA & AW &  \textbf{0.72} &   0.98 &  \textbf{0.55} &  355.3 \\
    & RW &  0.36 &   \textbf{1.00} &  0.09 &  400.0 \\
    & SW &  0.53 &   \textbf{1.00} &  0.33 &  \textbf{206.7} \\
    & JW &  0.46 &   0.75 &  0.20 &  239.9 \\
\bottomrule
\end{tabular}
    \label{tab:window-ablation}
\end{table}

\subsection{Runtime analysis}

\subsubsection{Comparison with competitors}

\Cref{fig:mtpo} shows the mean time per observation (MTPO) of ABCD and its competitors for $d\in\{10, 100, 1000, 10,000\}$ running single-threaded. 
The results are averaged over all evaluated parameters (\Cref{tab:hyperparameters}). \ac{name}~(id) replaces the encoder-decoder model with the identity which does not cause overhead. This allows measuring how much the encoder-decoder model influences ABCD's runtime. The results confirm that the runtime of ABCD alone, i.e, without the encoding-decoding-process, remains unaffected by a stream's dimensionality.

\begin{figure*}[htb]
\centering
    \begin{subfigure}[b]{\linewidth}
        \centering
        \includegraphics[width=\linewidth]{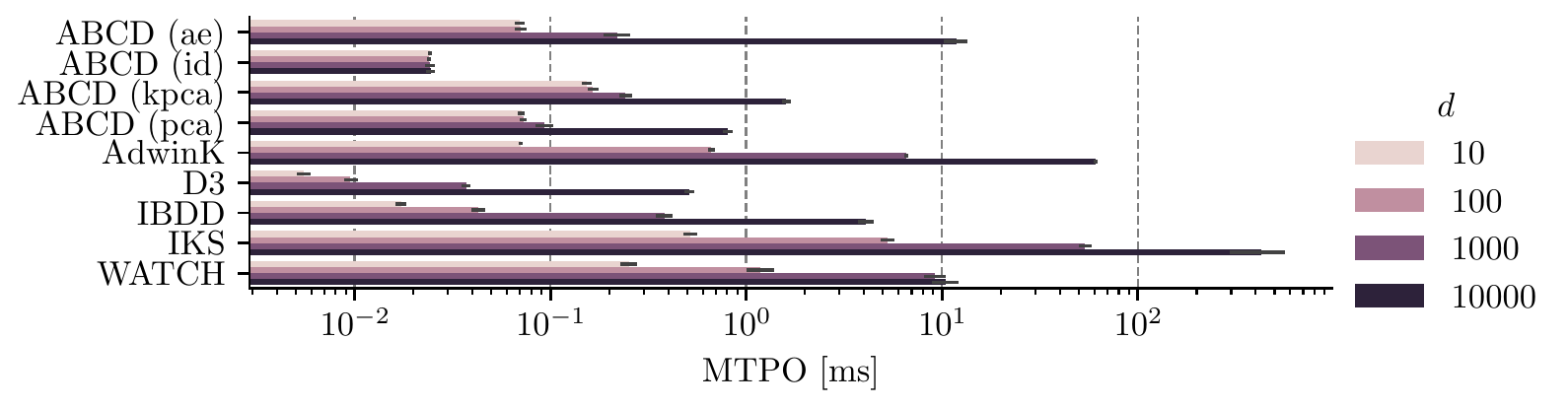}
        \caption{Mean time per observation in milliseconds.}
        \label{fig:mtpo}
    \end{subfigure}
    \begin{subfigure}[b]{\linewidth}
       \centering
        \includegraphics[width=\linewidth]{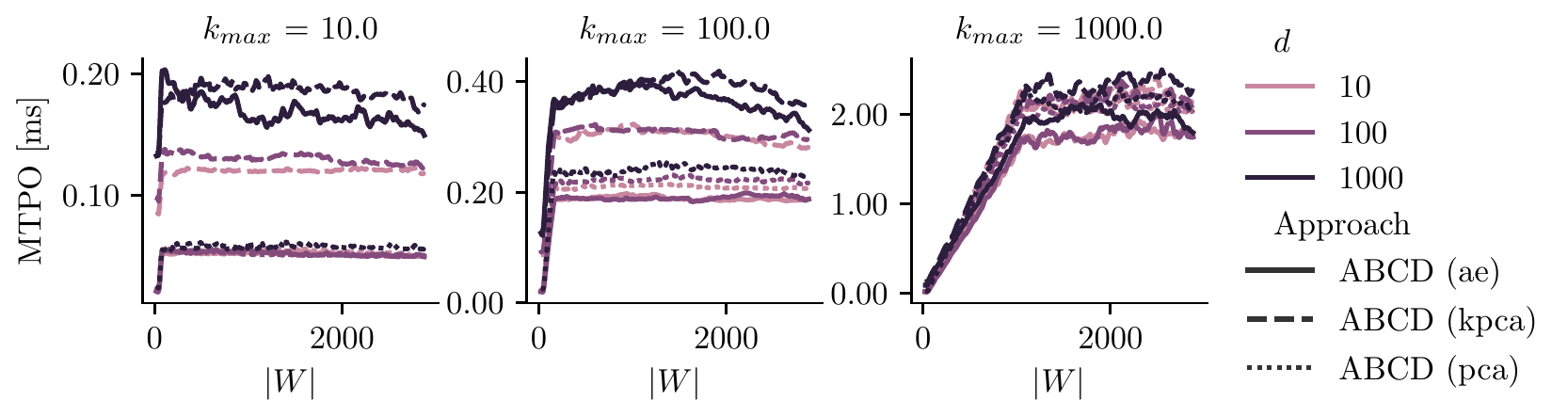}
        \caption{MTPO of ABCD over $\lvert W \rvert$ using $E=50$ and $\eta = 0.5$ varying $k_{max}$.}
        \label{fig:runtime-k}
    \end{subfigure}
    \begin{subfigure}[b]{\linewidth}
       \centering
        \includegraphics[width=\linewidth]{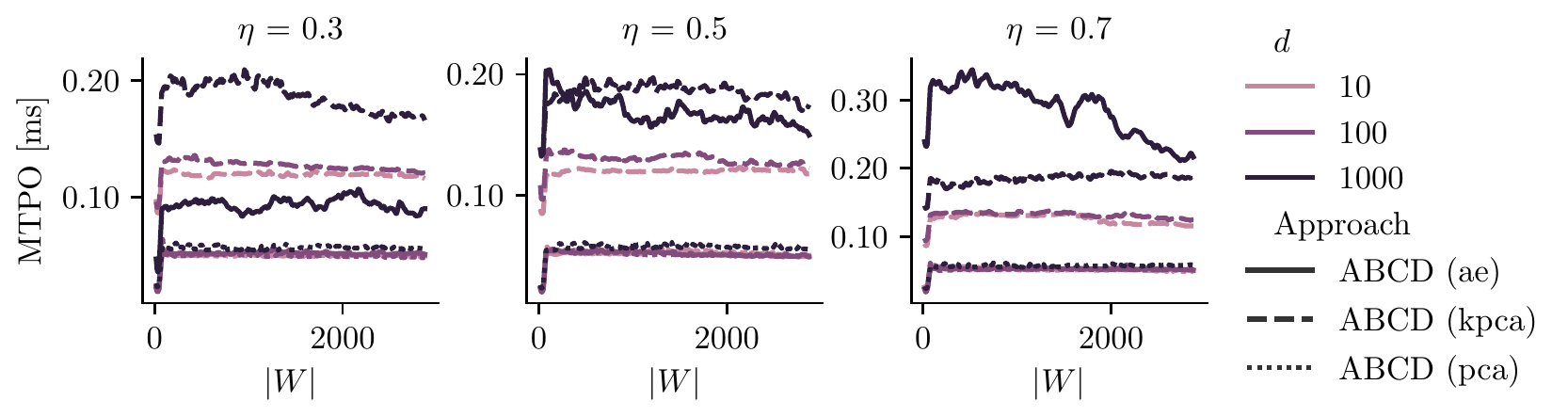}
        \caption{MTPO of ABCD over $\lvert W \rvert$ using $E=50$ and $k_{max} = 100$ varying $\eta$.}
        \label{fig:runtime-eta}
    \end{subfigure}
\caption{Runtime analysis of ABCD.}
\label{fig:runtime}
\end{figure*} 

We observe that our approach is able to process around 10,000 observations per second for $d\leq 100$. This is more than IKS, WATCH and AdwinK (except at $d=10$) but slower than D3 and IBDD. The reason is that our approach evaluates $k_{max}$ possible change points in each time step. In high-dimensional data, our competitors' MTPO grows faster than ABCD with PCA or KPCA; in fact, ABCD (pca) is second fastest after D3 for $d\geq 1000$. An exception is WATCH at $d=10000$. This is due to an iteration cap for approximating the Wasserstein distance restricting the approach's MTPO.

\subsubsection{Runtime depending on window size}

Next, we investigate \ac{name}'s runtime for different choices of $k_{max}$ and $\eta$. We run this experiment on a single CPU thread. For all three evaluated models, the encoding-decoding of an observation has a time complexity of $\mathcal{O}(\eta d^2)$; hence, ABCD's processing time of one instance is in $\mathcal{O}(\eta d^2 + k_{max})$. We therefore expect a quadratic increase in execution time with dimensionality and a linear increase with $\eta$ and $k_{max}$ when running on a single core.  

The results in \Cref{fig:runtime-k} show the influence of $k_{max}$ on the execution time: $k_{max}$ effectively restricts the MTPO as soon as $\lvert W \rvert = k_{max}$. Afterwards, MTPO remains unaffected by $\lvert W\rvert$. This also confirms that one can evaluate different possible change points in constant time using the proposed aggregates. 

We show the runtime for different choices of bottleneck-size $\eta$ in \Cref{fig:runtime-eta}. $\eta$ has little influence on the runtime of \ac{name} with PCA and Kernel-PCA
. However, coupled with an autoencoder (implemented in pytorch) we observe the expected linear increase in execution time from $0.1$\,ms for $\eta = 0.3$ to $0.3$\,ms for $\eta = 0.7$. Considering that change detection performance has shown to remain stable even for smaller choices of $\eta$, we recommend $\eta\leq 0.5$ as default. 

\section{Conclusion} \label{sec:conclusion}

We presented a change detector for high-dimen\-sio\-nal data streams, called \ac{name}, that monitors the reconstruction loss of an encoder-decoder-model in an adaptive window with a change score based on Bernstein's inequality. Our approach identifies changes and change subspaces, and provides a severity measure that correlates with the ground truth. 
\reviewerB
Since encoder-decoder models are already used in many domains~\citep{rani_BigDataDimensionality_2022}, our approach is widely applicable.
\reviewerA
In the future, it would thus be interesting to test ABCD with application or data specific encoder-decoder models. For example, one might observe even better performance on streams of image data when applying convolutional autoencoders. \erevision
Last, ABCD could also benefit from a theoretical analysis of the relationship between changes in data distribution and the loss of different encoder-decoder models. 

\backmatter

\ifpreprint
\else
\section*{Statements and Declarations}

The authors have no competing interests to declare that are relevant to the content of this article.
\fi

\bibliography{references.bib}

\begin{thebibliography}{60}
\providecommand{\natexlab}[1]{#1}
\providecommand{\url}[1]{{#1}}
\providecommand{\urlprefix}{URL }
\providecommand{\doi}[1]{\url{https://doi.org/#1}}
\providecommand{\eprint}[2][]{\url{#2}}
 \bibcommenthead

\bibitem[{Anguita et~al(2013)Anguita, Ghio, Oneto, Parra, and
  Reyes{-}Ortiz}]{anguita_public_2013}
Anguita D, Ghio A, Oneto L, et~al (2013) A public domain dataset for human
  activity recognition using smartphones. In: {ESANN},
  \urlprefix\url{https://www.esann.org/sites/default/files/proceedings/legacy/es2013-84.pdf}

\bibitem[{Bai and Perron(2003)}]{bai_critical_2003}
Bai J, Perron P (2003) Critical values for multiple structural change tests.
  The Econometrics Journal 6(1):72--78.
  \doi{https://doi.org/10.1111/1368-423X.00102}

\bibitem[{Bernstein(1924)}]{bernstein_modification_1924}
Bernstein SN (1924) On a modification of {Chebyshev}'s inequality and of the
  error formula of {Laplace}. Ann. Sci. Inst. Savantes Ukraine, Sect. Math.

\bibitem[{Bifet and Gavald{\`{a}}(2007)}]{bifet_learning_2007}
Bifet A, Gavald{\`{a}} R (2007) Learning from time-changing data with adaptive
  windowing. In: Proceedings of the Seventh {SIAM} International Conference on
  Data Mining. {SIAM}, pp 443--448, \doi{10.1137/1.9781611972771.42}

\bibitem[{Bifet et~al(2010)Bifet, Holmes, and
  Pfahringer}]{bifet_leveraging_2010}
Bifet A, Holmes G, Pfahringer B (2010) Leveraging bagging for evolving data
  streams. In: {ECML} {PKDD}, Lecture Notes in Computer Science, vol 6321.
  Springer, pp 135--150, \doi{10.1007/978-3-642-15880-3\_15}

\bibitem[{Boucheron et~al(2013)Boucheron, Lugosi, and
  Massart}]{boucheron_concentration_2013}
Boucheron S, Lugosi G, Massart P (2013) {Concentration Inequalities: A
  Nonasymptotic Theory of Independence}. Oxford University Press,
  \doi{10.1093/acprof:oso/9780199535255.001.0001}

\bibitem[{Ceci et~al(2020)Ceci, Corizzo, Japkowicz, Mignone, and
  Pio}]{ceci_ECHADEmbeddingbasedChange_2020}
Ceci M, Corizzo R, Japkowicz N, et~al (2020) {ECHAD}: {Embedding}-based change
  detection from multivariate time series in smart grids. IEEE Access
  8:156,053--156,066. \doi{10.1109/ACCESS.2020.3019095}

\bibitem[{Chakar et~al(2017)Chakar, Lebarbier, L{\'e}vy-Leduc, and
  Robin}]{chakar_robust_2017}
Chakar S, Lebarbier E, L{\'e}vy-Leduc C, et~al (2017) {A robust approach for
  estimating change-points in the mean of an $\operatorname{AR}(1)$ process}.
  Bernoulli 23(2):1408 -- 1447. \doi{10.3150/15-BEJ782}

\bibitem[{Chan et~al(1982)Chan, Golub, and LeVeque}]{chan_updating_1982}
Chan TF, Golub GH, LeVeque RJ (1982) Updating formulae and a pairwise algorithm
  for computing sample variances. Tech. rep., Heidelberg

\bibitem[{Chaudhuri et~al(2021)Chaudhuri, Fellouris, and
  Tajer}]{chaudhuri_sequential_2021}
Chaudhuri A, Fellouris G, Tajer A (2021) Sequential change detection of a
  correlation structure under a sampling constraint. In: {ISIT}, pp 605--610,
  \doi{10.1109/ISIT45174.2021.9517736}

\bibitem[{Dasu et~al(2006)Dasu, Krishnan, Venkatasubramanian, and
  Yi}]{dasu_information-theoretic_2006}
Dasu T, Krishnan S, Venkatasubramanian S, et~al (2006) An information-theoretic
  approach to detecting changes in multi-dimensional data streams. In: Proc.
  Symposium on the Interface of Statistics, Computing Science, and Applications
  (Interface)

\bibitem[{Faber et~al(2021)Faber, Corizzo, Sniezynski, Baron, and
  Japkowicz}]{faber_watch_2021}
Faber K, Corizzo R, Sniezynski B, et~al (2021) {WATCH:} {Wasserstein} change
  point detection for high-dimensional time series data. In: {Big Data}.
  {IEEE}, pp 4450--4459, \doi{10.1109/BigData52589.2021.9671962}

\bibitem[{Faithfull et~al(2019)Faithfull, Diez, and
  Kuncheva}]{faithfull_combining_2019}
Faithfull WJ, Diez JJR, Kuncheva LI (2019) Combining univariate approaches for
  ensemble change detection in multivariate data. Inf Fusion 45:202--214.
  \doi{10.1016/j.inffus.2018.02.003}

\bibitem[{Fouch{\'{e}} et~al(2019)Fouch{\'{e}}, Komiyama, and
  B{\"{o}}hm}]{fouche_scaling_2019}
Fouch{\'{e}} E, Komiyama J, B{\"{o}}hm K (2019) Scaling multi-armed bandit
  algorithms. In: {SIGKDD}. {ACM}, pp 1449--1459, \doi{10.1145/3292500.3330862}

\bibitem[{Garreau and Arlot(2018)}]{garreau_consistent_2018}
Garreau D, Arlot S (2018) {Consistent change-point detection with kernels}.
  {Electronic Journal of Statistics} 12(2):4440--4486.
  \urlprefix\url{https://hal.science/hal-01416704}

\bibitem[{Goldenberg and Webb(2019)}]{goldenberg_survey_2019}
Goldenberg I, Webb GI (2019) Survey of distance measures for quantifying
  concept drift and shift in numeric data. Knowl Inf Syst 60(2):591--615.
  \doi{10.1007/s10115-018-1257-z}

\bibitem[{G\"{o}z\"{u}a\c{c}\i{}k et~al(2019)G\"{o}z\"{u}a\c{c}\i{}k,
  B\"{u}y\"{u}k\c{c}ak\i{}r, Bonab, and Can}]{gozuacik_unsupervised_2019}
G\"{o}z\"{u}a\c{c}\i{}k O, B\"{u}y\"{u}k\c{c}ak\i{}r A, Bonab H, et~al (2019)
  Unsupervised concept drift detection with a discriminative classifier. In:
  CIKM. {ACM}, p 2365–2368, \doi{10.1145/3357384.3358144}

\bibitem[{Gretton et~al(2012)Gretton, Borgwardt, Rasch, Sch{\"{o}}lkopf, and
  Smola}]{gretton_kernel_2012}
Gretton A, Borgwardt KM, Rasch MJ, et~al (2012) A kernel two-sample test. J
  Mach Learn Res 13:723--773. \doi{10.5555/2503308.2188410}

\bibitem[{Harchaoui and Cappe(2007)}]{harchaoui_retrospective_2007}
Harchaoui Z, Cappe O (2007) Retrospective mutiple change-point estimation with
  kernels. In: {IEEE/SP} 14th Workshop on Statistical Signal Processing, pp
  768--772, \doi{10.1109/SSP.2007.4301363}

\bibitem[{Impedovo et~al(2019)Impedovo, Ceci, and
  Calders}]{impedovo_EfficientAccurateNonexhaustive_2019}
Impedovo A, Ceci M, Calders T (2019) Efficient and accurate non-exhaustive
  pattern-based change detection in dynamic networks. Lecture notes in computer
  science, vol 11828. Springer, pp 396--411,
  \doi{10.1007/978-3-030-33778-0\_30}

\bibitem[{Impedovo et~al(2020{\natexlab{a}})Impedovo, Loglisci, Ceci, and
  Malerba}]{impedovo_CondensedRepresentationsChanges_2020}
Impedovo A, Loglisci C, Ceci M, et~al (2020{\natexlab{a}}) Condensed
  representations of changes in dynamic graphs through emerging subgraph
  mining. Engineering Applications of Artificial Intelligence 94:103,830.
  \doi{10.1016/j.engappai.2020.103830}

\bibitem[{Impedovo et~al(2020{\natexlab{b}})Impedovo, Mignone, Loglisci, and
  Ceci}]{impedovo_SimultaneousProcessDrift_2020}
Impedovo A, Mignone P, Loglisci C, et~al (2020{\natexlab{b}}) Simultaneous
  process drift detection and characterization with pattern-based change
  detectors. Lecture notes in computer science, vol 12323. Springer, pp
  451--467, \doi{10.1007/978-3-030-61527-7\_30}

\bibitem[{Iwashita and Papa(2019)}]{iwashita_overview_2019}
Iwashita AS, Papa JP (2019) An overview on concept drift learning. {IEEE}
  Access 7:1532--1547. \doi{10.1109/ACCESS.2018.2886026}

\bibitem[{Jaworski et~al(2020)Jaworski, Rutkowski, and
  Angelov}]{jaworski_concept_2020}
Jaworski M, Rutkowski L, Angelov P (2020) Concept drift detection using
  autoencoders in data streams processing. In: {ICAISC}. Springer-Verlag, p
  124–133, \doi{10.1007/978-3-030-61401-0_12}

\bibitem[{Jiao et~al(2018)Jiao, Chen, and Gu}]{jiao_subspace_2018}
Jiao Y, Chen Y, Gu Y (2018) Subspace change-point detection: A new model and
  solution. IEEE Journal of Selected Topics in Signal Processing
  12(6):1224--1239. \doi{10.1109/JSTSP.2018.2873147}

\bibitem[{de~Jong and Bosman(2019)}]{dejong_UnsupervisedChangeDetection_2019}
de~Jong KL, Bosman AS (2019) Unsupervised change detection in satellite images
  using convolutional neural networks. In: {IJCNN} 2019. IEEE, pp 1--8,
  \doi{10.1109/IJCNN.2019.8851762}

\bibitem[{Khamassi et~al(2015)Khamassi, {Sayed Mouchaweh}, Hammami, and
  Gh{\'{e}}dira}]{khamassi_self-adaptive_2015}
Khamassi I, {Sayed Mouchaweh} M, Hammami M, et~al (2015) Self-adaptive
  windowing approach for handling complex concept drift. Cogn Comput
  7(6):772--790. \doi{10.1007/s12559-015-9341-0}

\bibitem[{Killick et~al(2012)Killick, Fearnhead, and
  Eckley}]{killick_optimal_2012}
Killick R, Fearnhead P, Eckley IA (2012) Optimal detection of changepoints with
  a linear computational cost. Journal of the American Statistical Association
  107(500):1590--1598. \doi{10.1080/01621459.2012.737745}

\bibitem[{Kingma and Ba(2015)}]{kingma_adam_2015}
Kingma DP, Ba J (2015) Adam: {A} method for stochastic optimization. In:
  {ICLR}, \urlprefix\url{http://arxiv.org/abs/1412.6980}

\bibitem[{Knuth(1997)}]{knuth_art_1997}
Knuth DE (1997) The Art of Computer Programming: Seminumerical Algorithms,
  vol~2. Addison-Wesley

\bibitem[{Krizhevsky et~al(2009)Krizhevsky, Hinton
  et~al}]{krizhevsky_learning_2009}
Krizhevsky A, Hinton G, et~al (2009) Learning multiple layers of features from
  tiny images. Tech. rep.

\bibitem[{Lajugie et~al(2014)Lajugie, Bach, and
  Arlot}]{lajugie_large-margin_2014}
Lajugie R, Bach FR, Arlot S (2014) Large-margin metric learning for constrained
  partitioning problems. In: {ICML}, {JMLR} Workshop and Conference
  Proceedings, vol~32. JMLR.org, pp 297--305,
  \urlprefix\url{http://proceedings.mlr.press/v32/lajugie14.html}

\bibitem[{LeCun et~al(1998)LeCun, Cortes, and Burges}]{lecun_mnist_2010}
LeCun Y, Cortes C, Burges C (1998) The {MNIST} database of handwritten digits.
  Retrieved from \url{http://yann.lecun.com/exdb/mnist/}

\bibitem[{Lee and Verleysen(2007)}]{lee_nonlinear_2007}
Lee JA, Verleysen M (2007) Nonlinear {Dimensionality} {Reduction}. Springer,
  \doi{10.1007/978-0-387-39351-3}

\bibitem[{Liu et~al(2017)Liu, Song, Zhang, and Lu}]{liu_regional_2017}
Liu A, Song Y, Zhang G, et~al (2017) Regional concept drift detection and
  density synchronized drift adaptation. In: {IJCAI}. ijcai.org, pp 2280--2286,
  \doi{10.24963/ijcai.2017/317}

\bibitem[{Liu et~al(2019)Liu, Wang, Wang, Zhang, Wang, and
  Dorrell}]{liu_HighdimensionalDataAbnormity_2019}
Liu P, Wang J, Wang Z, et~al (2019) High-dimensional data abnormity detection
  based on improved {Variance}-of-{Angle} ({VOA}) algorithm for electric
  vehicles battery. In: 2019 {IEEE} energy conversion congress and exposition
  ({ECCE}), pp 5072--5077, \doi{10.1109/ECCE.2019.8912777}

\bibitem[{Loglisci et~al(2018)Loglisci, Ceci, Impedovo, and
  Malerba}]{loglisci_MiningMicroscopicMacroscopic_2018}
Loglisci C, Ceci M, Impedovo A, et~al (2018) Mining microscopic and macroscopic
  changes in network data streams. Knowl Based Syst 161:294--312.
  \doi{10.1016/j.knosys.2018.07.011}

\bibitem[{Lu et~al(2019)Lu, Liu, Dong, Gu, Gama, and Zhang}]{lu_learning_2019}
Lu J, Liu A, Dong F, et~al (2019) Learning under concept drift: {A} review.
  {IEEE} Trans Knowl Data Eng 31(12):2346--2363.
  \doi{10.1109/TKDE.2018.2876857}

\bibitem[{Lung-Yut-Fong et~al(2015)Lung-Yut-Fong, L{\'e}vy-Leduc, and
  Capp{\'e}}]{lung-yut-fong_homogeneity_2015}
Lung-Yut-Fong A, L{\'e}vy-Leduc C, Capp{\'e} O (2015) Homogeneity and
  change-point detection tests for multivariate data using rank statistics.
  Journal de la Soci{\'e}t{\'e} Fran{\c{c}}aise de Statistique 156(4):133--162

\bibitem[{Matteson and James(2014)}]{matteson_nonparametric_2014}
Matteson DS, James NA (2014) A nonparametric approach for multiple change point
  analysis of multivariate data. Journal of the American Statistical
  Association 109(505):334--345. \doi{10.1080/01621459.2013.849605}

\bibitem[{Montiel et~al(2018)Montiel, Read, Bifet, and
  Abdessalem}]{montiel_scikit-multiflow_2018}
Montiel J, Read J, Bifet A, et~al (2018) Scikit-multiflow: A multi-output
  streaming framework. JMLR 19(1):2915–2914.
  \urlprefix\url{http://jmlr.org/papers/v19/18-251.html}

\bibitem[{Mowbray et~al(2021)Mowbray, Savage, Wu, Song, Cho, Del Rio-Chanona,
  and Zhang}]{mowbray_MachineLearningBiochemical_2021}
Mowbray M, Savage T, Wu C, et~al (2021) Machine learning for biochemical
  engineering: {A} review. Biochemical Engineering Journal 172:108,054.
  \urlprefix\url{https://www.sciencedirect.com/science/article/pii/S1369703X21001303}

\bibitem[{Naseer et~al(2020)Naseer, Ali, Dominic, and
  Saleem}]{naseer_LearningRepresentationsNetwork_2020}
Naseer S, Ali RF, Dominic PDD, et~al (2020) Learning representations of network
  traffic using deep neural networks for network anomaly detection: {A}
  perspective towards oil and gas {IT} infrastructures. Symmetry 12(11):1882.
  \doi{10.3390/sym12111882}

\bibitem[{Page(1954)}]{page_continuous_1954}
Page ES (1954) Continuous inspection schemes. Biometrika 41(1-2):100--115.
  \doi{10.1093/biomet/41.1-2.100}

\bibitem[{Pears et~al(2014)Pears, Sakthithasan, and Koh}]{pears_detecting_2014}
Pears R, Sakthithasan S, Koh YS (2014) Detecting concept change in dynamic data
  streams. Machine Learning 97(3):259--293. \doi{10.1007/s10994-013-5433-9}

\bibitem[{Qahtan et~al(2015)Qahtan, Alharbi, Wang, and
  Zhang}]{qahtan_pca-based_2015}
Qahtan AA, Alharbi B, Wang S, et~al (2015) A {PCA}-based change detection
  framework for multidimensional data streams: {Change} detection in
  multidimensional data streams. In: SIGKDD. ACM, New York, NY, USA, p
  935–944, \doi{10.1145/2783258.2783359}

\bibitem[{Rani et~al(2022)Rani, Khurana, Kumar, and
  Kumar}]{rani_BigDataDimensionality_2022}
Rani R, Khurana M, Kumar A, et~al (2022) Big data dimensionality reduction
  techniques in {IoT}: review, applications and open research challenges.
  Cluster Computing 25(6):4027--4049. \doi{10.1007/s10586-022-03634-y}

\bibitem[{dos Reis et~al(2016)dos Reis, Flach, Matwin, and
  Batista}]{dos_reis_fast_2016}
dos Reis DM, Flach PA, Matwin S, et~al (2016) Fast unsupervised online drift
  detection using incremental {Kolmogorov-Smirnov} test. In: {SIGKDD}. {ACM},
  pp 1545--1554, \doi{10.1145/2939672.2939836}

\bibitem[{Shewhart(1930)}]{shewhart_economic_1931}
Shewhart WA (1930) Economic Quality Control of Manufactured Product, vol~9.
  \doi{https://doi.org/10.1002/j.1538-7305.1930.tb00373.x}

\bibitem[{de~Souza et~al(2020)de~Souza, Chowdhury, and
  Mueen}]{souza_unsupervised_2020}
de~Souza VMA, Chowdhury FA, Mueen A (2020) Unsupervised drift detection on
  high-speed data streams. In: BigData. {IEEE}, pp 102--111,
  \doi{10.1109/BigData50022.2020.9377880}

\bibitem[{Sun et~al(2016)Sun, Wang, Liu, Du, and Yuan}]{sun_online_2016}
Sun Y, Wang Z, Liu H, et~al (2016) Online ensemble using adaptive windowing for
  data streams with concept drift. Int J Distributed Sens Networks
  12(5):4218,973:1--4218,973:9. \doi{10.1155/2016/4218973}

\bibitem[{Suryawanshi et~al(2022)Suryawanshi, Goswami, Patil, and
  Mishra}]{suryawanshi_adaptive_2022}
Suryawanshi S, Goswami A, Patil P, et~al (2022) Adaptive windowing based
  recurrent neural network for drift adaption in non-stationary environment.
  Journal of Ambient Intelligence and Humanized Computing
  \doi{10.1007/s12652-022-04116-0}

\bibitem[{Truong et~al(2020)Truong, Oudre, and Vayatis}]{truong_selective_2020}
Truong C, Oudre L, Vayatis N (2020) Selective review of offline change point
  detection methods. Signal Process 167. \doi{10.1016/j.sigpro.2019.107299}

\bibitem[{Vergara et~al(2011)Vergara, Huerta, Ayhan, Ryan, Vembu, and
  Homer}]{vergara_gas_2011}
Vergara A, Huerta R, Ayhan T, et~al (2011) Gas sensor drift mitigation using
  classifier ensembles. In: Proceedings of the Fifth International Workshop on
  Knowledge Discovery from Sensor Data. {ACM}, SensorKDD '11, p 16–24,
  \doi{10.1145/2003653.2003655}

\bibitem[{Vrigkas et~al(2015)Vrigkas, Nikou, and
  Kakadiaris}]{vrigkas_ReviewHumanActivity_2015}
Vrigkas M, Nikou C, Kakadiaris IA (2015) A review of human activity recognition
  methods. Frontiers Robotics AI 2:28. \doi{10.3389/frobt.2015.00028}

\bibitem[{Webb et~al(2018)Webb, Lee, Goethals, and
  Petitjean}]{webb_analyzing_2018}
Webb GI, Lee LK, Goethals B, et~al (2018) Analyzing concept drift and shift
  from sample data. Data Min Knowl Discov 32(5):1179--1199.
  \doi{10.1007/s10618-018-0554-1}

\bibitem[{Xiao et~al(2017)Xiao, Rasul, and Vollgraf}]{xiao_fashion-mnist_2017}
Xiao H, Rasul K, Vollgraf R (2017) Fashion-{MNIST}: a novel image dataset for
  benchmarking machine learning algorithms. CoRR abs/1708.07747.
  \urlprefix\url{http://arxiv.org/abs/1708.07747}

\bibitem[{Xie et~al(2020)Xie, Xie, and Moustakides}]{xie_sequential_2020}
Xie L, Xie Y, Moustakides GV (2020) Sequential subspace change point detection.
  Sequential Analysis 39(3):307--335. \doi{10.1080/07474946.2020.1823191}

\bibitem[{Yan et~al(2006)Yan, Zhang, Liu, Yan, Cheng, Fan, Yang, Xi, and
  Chen}]{yan_effective_2006}
Yan J, Zhang B, Liu N, et~al (2006) Effective and efficient dimensionality
  reduction for large-scale and streaming data preprocessing. IEEE Transactions
  on Knowledge and Data Engineering 18(3):320--333. \doi{10.1109/TKDE.2006.45}

\bibitem[{Zhao et~al(2018)Zhao, Wu, Shi, Wang, and
  Zhang}]{xiaopingzhaojiaxinwu_FaultDiagnosisMotor_2018}
Zhao X, Wu J, Shi Y, et~al (2018) Fault diagnosis of motor in frequency domain
  signal by stacked de-noising auto-encoder. Computers, Materials \& Continua
  57(2):223--242. \urlprefix\url{http://www.techscience.com/cmc/v57n2/22968}

\end{thebibliography}

\clearpage
\appendix
\section{Appendix}

\subsection{Training of autoencoder}\label{app:ae-training}

\Cref{alg:ae-training} describes the training of the autoencoder model as done in our experiments. First, we collect the training data from the current window $W$ (line~2). Afterwards we perform gradient descent on $X_{train}$ for $E$ epochs at a learning rate of $lr$.

\begin{algorithm}[htb]
\begin{algorithmic}[1]
\caption{Autoencoder training}
\label{alg:ae-training}
\Require{$W$, learning rate $lr$, number of training epochs $E$}
\Procedure{\textsc{TrainAE}}{$W, lr, E$}
\State $X_{train}\gets \{x_i\ \forall\ (-, -, x_i) \in W\}$
\State $\phi,\psi\gets$ \textsc{NewEncoder}(), \textsc{NewDecoder}()
\ForAll{epochs $E$}
    \State \textsc{GradientDescent}$(\psi\circ\phi, X_{train}, lr)$
\EndFor
\State Return $\phi,\psi$
\EndProcedure
\end{algorithmic}
\end{algorithm}

\subsection{Detectable and undetectable change for ABCD (pca)}\label{app:detectable-undetectable}

This section illustrates under which conditions one can use principal component analysis to detect change. \Cref{fig:detectable-undetectable} shows data from two distributions: black points (e.g., before the change) plus the associated main principle component, and blue points (e.g., after the change). On the left, the change affects the correlation between Dim.~1 and Dim.~2. This leads to an increased reconstruction error for the points highlighted in blue. On the right, the change occurs along the main principle component. I.e., the variance along the main principle component has increased. Such kind of change is undetectable by ABCD~(pca) as the reconstruction error remains unchanged. 

\begin{figure}[htb]
    \centering
    \includegraphics[width=.93\linewidth]{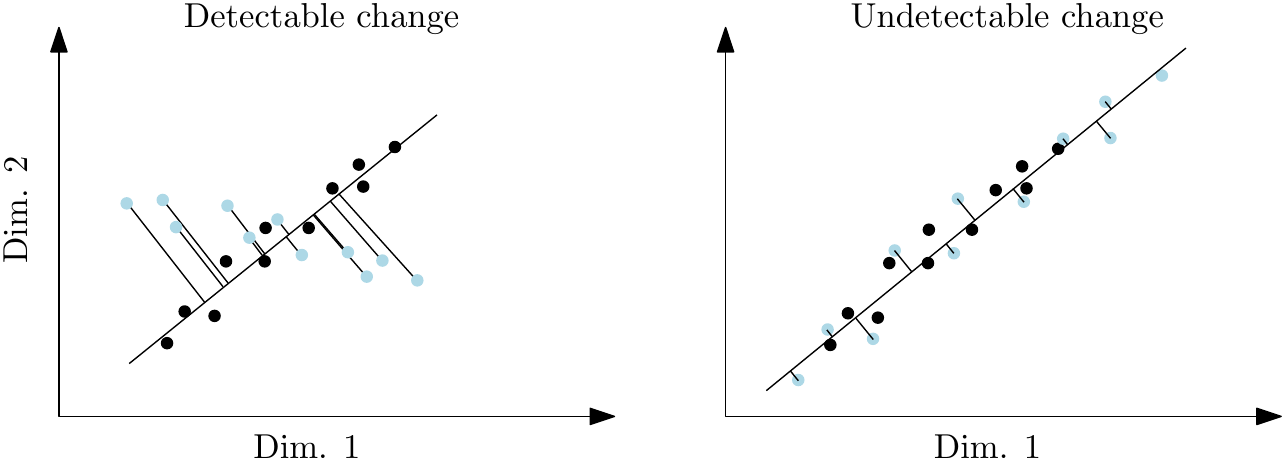}
    \caption{Illustration of detectable and undetectable change using PCA.}
    \label{fig:detectable-undetectable}
\end{figure}

\subsection{Reconstruction loss over time}\label{app:reconstruction-loss}

\Cref{fig:reconstruction-loss} shows the reconstruction loss of the evaluated encoder-decoder models over the length of the stream. We observe that indeed the reconstruction loss decreases with increasing bottleneck size (controlled by $\eta$), and with increasing number of training epochs $E$ (first three columns). Further, we see that regardless of $E$, $\eta$, or the type of model, the reconstruction loss typically changes after a change point. After the change was detected, ABCD learns the new concept, which mostly leads to a decrease in reconstruction loss. Last, we observe that the theoretical limit of $M=1$ for the absolute difference between the reconstruction loss and its expected value is overly conservative. A value of $M=0.1$ seems to be a more realistic choice. 

\begin{figure}[htb]
    \centering
    \includegraphics[width=\linewidth]{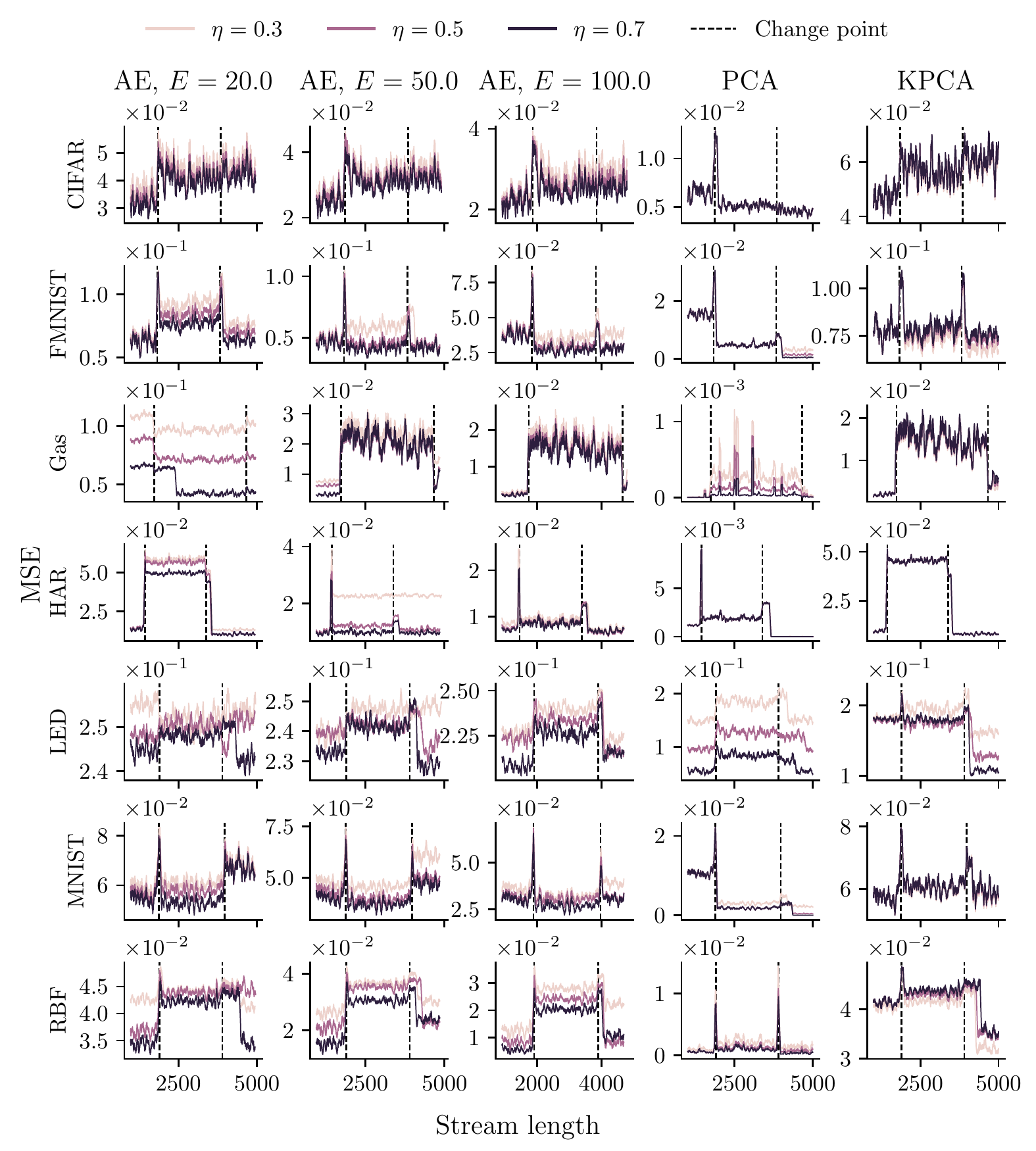}
    \caption{Reconstruction loss over the length of the stream.}
    \label{fig:reconstruction-loss}
\end{figure}

\end{document}